\documentclass[nohyperref]{article}

\usepackage{microtype}
\usepackage{graphicx}
\usepackage{subcaption}
\usepackage{booktabs} 

\usepackage{hyperref}

\usepackage[accepted]{icml2022}

\usepackage{amsmath}
\usepackage{amssymb}
\usepackage{mathtools}
\usepackage{amsthm}

\usepackage{graphicx, bbm, multirow, listings}
\usepackage{enumitem}
\usepackage{float}
\let\vec\mathbf
\usepackage{soul}

\newcommand{\data}{{\vec{X}}}

\newcommand{\weight}{\vec{w}}
\newcommand{\weightmat}{\vec{W}}

\newcommand{\labelmat}{\vec{Y}}

\newcommand{\dualmat}{\vec{V}}

\newcommand{\vecop}{\mathbf{vec}}

\usepackage[capitalize,noabbrev]{cleveref}

\theoremstyle{plain}
\newtheorem{theorem}{Theorem}[section]

\newtheorem{lemma}[theorem]{Lemma}
\newtheorem{corollary}[theorem]{Corollary}
\theoremstyle{definition}

\theoremstyle{remark}

\usepackage[textsize=tiny]{todonotes}

\icmltitlerunning{Unraveling Attention via Convex Duality}
\usepackage[toc,page,header]{appendix}
\usepackage{minitoc}

\begin{document}
\doparttoc 
\faketableofcontents 

\twocolumn[
\icmltitle{Unraveling Attention via Convex Duality: \\ 
            Analysis and Interpretations of Vision Transformers}




\begin{icmlauthorlist}
\icmlauthor{Arda Sahiner}{stanford}
\icmlauthor{Tolga Ergen}{stanford}
\icmlauthor{Batu Ozturkler}{stanford}
\icmlauthor{John Pauly}{stanford}
\icmlauthor{Morteza Mardani}{nvidia}
\icmlauthor{Mert Pilanci}{stanford}
\end{icmlauthorlist}

\icmlaffiliation{stanford}{Department of Electrical Engineering, Stanford University, Stanford, CA, USA}
\icmlaffiliation{nvidia}{NVIDIA Corporation, Santa Clara, CA, USA}

\icmlcorrespondingauthor{Arda Sahiner}{sahiner@stanford.edu}

\icmlkeywords{Machine Learning, ICML}

\vskip 0.3in
]



\printAffiliationsAndNotice{} 

\begin{abstract}
Vision transformers using self-attention or its proposed alternatives have demonstrated promising results in many image related tasks. However, the underpinning inductive bias of attention is not well understood. To address this issue, this paper analyzes attention through the lens of convex duality. For the non-linear dot-product self-attention, and alternative mechanisms such as MLP-mixer and Fourier Neural Operator (FNO), we derive equivalent finite-dimensional convex problems that are interpretable and solvable to global optimality. The convex programs lead to {\it block nuclear-norm regularization} that promotes low rank in the latent feature and token dimensions. In particular, we show how self-attention networks implicitly clusters the tokens, based on their latent similarity. We conduct experiments for transferring a pre-trained transformer backbone for CIFAR-100 classification by fine-tuning a variety of convex attention heads. The results indicate the merits of the bias induced by attention compared with the existing MLP or linear heads. 
%
%
\end{abstract}

\section{Introduction}
Transformers have recently delivered tremendous success for representation learning in language and vision. This is primarily due to the attention mechanism that effectively mixes the tokens’ representation over the layers to learn the semantics present in the input\footnote{We use the terms ``attention", ``token mixing", and ``mixing" interchangeably throughout this paper.}. After the inception of dot-product self-attention \cite{vaswani2017attention}, there have been several efficient alternatives that scale nicely with the sequence size for large pretraining tasks; see e.g., \cite{wang2020linformer,shen2021efficient,kitaev2020reformer, panahi2021shapeshifter, xiong2021nystr}. However, the learnable inductive bias of attention is not explored well. A strong theoretical understanding of attention's inductive bias can motivate designing more efficient architectures, and can explain the generalization ability of these networks.

Self-attention was the fundamental building block in the first proposed vision transformer (ViT) \cite{dosovitskiy2020image}. It consists of an outer product of two linear functions, followed by a non-linearity and a product with another linear function, which makes it non-convex and non-interpretable. One approach to understand attention has been to design new alternatives to self-attention which perform similarly well, which may help explain its underlying mechanisms. One set of work pertains to Multi-Layer Perceptron (MLP) based architectures, \cite{tolstikhin2021mlp, tatsunami2021raftmlp, touvron2021resmlp, liu2021pay, yu2021metaformer}, while  another line of work proposes Fourier based models \cite{lee2021fnet, rao2021global, li2020fourier, guibas2021adaptive}. Others have proposed replacing self-attention with matrix decomposition \cite{geng2021attention}. While all of these works have appealing applications that leverage general concepts about the structure of attention, they lack any fine-tuned theoretical analysis on these architectures from an optimization perspective.

To address this shortcoming, we leverage convex duality to analyze a single block of self-attention with ReLU activation. Since self-attention incurs quadratic complexity in the sequence, we alternatively analyze more efficient modules. As representatives for more efficient modules, we focus on MLP Mixers \cite{tolstikhin2021mlp} and Fourier Neural Operators (FNO) \cite{li2020fourier}. MLP-mixer mixes tokens (purely) using MLP projections in both the token and feature dimensions. In contrast, FNO mixes tokens using circular convolution that is efficiently implemented based on 2D Fourier transforms.  
\begin{figure*}[t!]
  \centering
  \begin{center}
      \includegraphics[width = \linewidth]{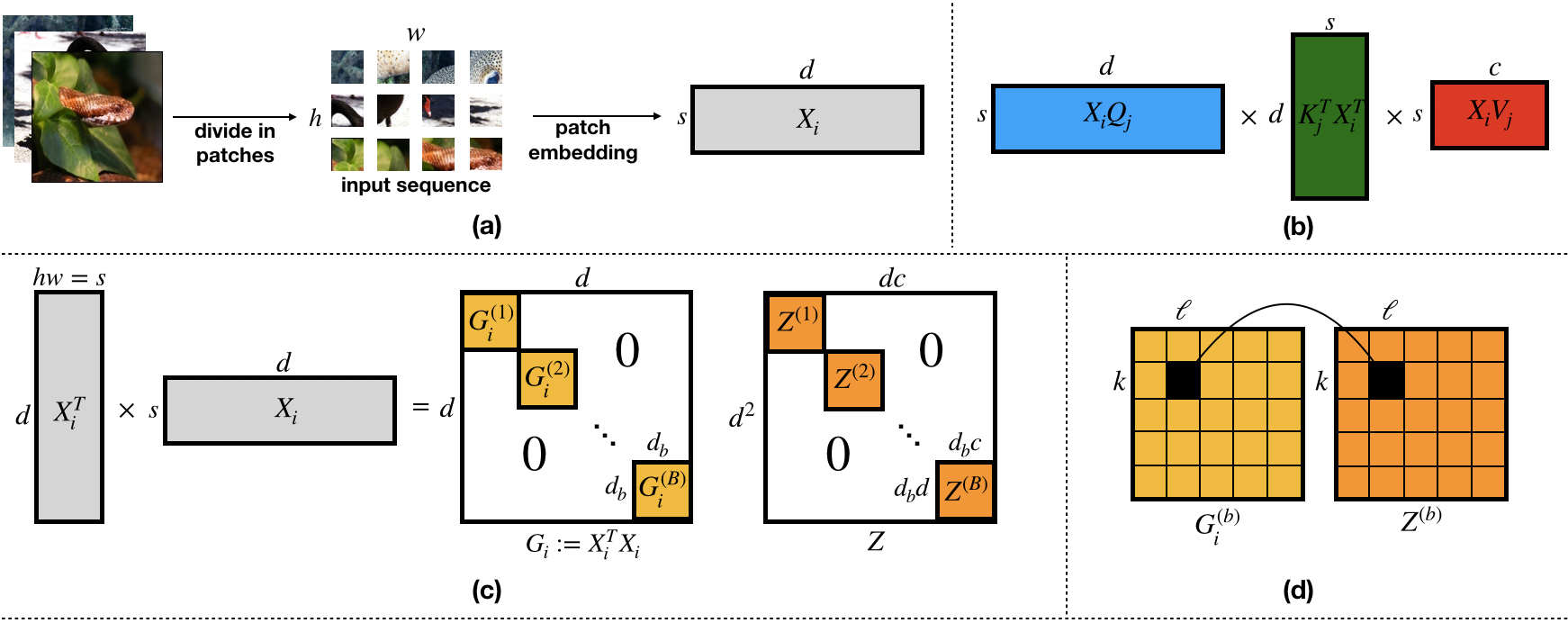}
  \end{center}
  \caption{\small Illustration of the implicit convexity of (linear activation) multi-head self-attention. (a) Input image is first divided into $hw = s$ patches, where each patch is represented by a latent vector of dimension $d$. (b) The (non-convex) scaled dot-product  self-attention applies learnable weights $\vec{Q}_j$, $\vec{K}_j$, $\vec{V}_j$ to the patch embeddings $X_i$ as in \eqref{eq:mhsa_original}. (c) In the equivalent convex optimization problem for the self-attention training objective \eqref{eq:mhsa_linear_convex_block_diag}, the Gram matrix $G_i$ is formed that groups latent features in $B$ different blocks, (d) and accordingly the nuclear norm regularization is imposed on the dual variables $\vec{Z}^{(b,k,l)} \in \mathbb{R}^{d \times c}$ based on the similarity scores $\vec{G}_i^{(b)}[k,l]$.}
  \label{fig:cvx-sa}
\end{figure*}

We find that all three of these analyzed modules are equivalent to finite-dimensional convex optimization problems, indicating that there are guarantees to provably optimize them to their global optima. Furthermore, we make novel observations about the bias induced by the convex models. In particular, convexified equivalents of both self-attention and MLP-Mixer modules resemble weighted combinations of MLPs, but with additional degrees of freedom (e.g. higher dimensional induced parameters), and a unique block nuclear norm regularization which ties their individual sub-modules together to encourage global cohesiveness. In contrast, the convexified FNO mixer amounts to circular convolution, while slight modifications to the FNO architecture can induce an equivalent group-wise convolution. We experimentally test and compare these convex attention heads for transfer learning of a pre-trained vision transformer to CIFAR-100 classification upon finetuning a single convex head. We observe that the inductive bias of these attention modules outperforms traditional convex models.

Our main contributions are summarized as follows:
\begin{itemize}
\item We provide guarantees that self-attention, MLP-Mixer, and FNO with linear and ReLU activation can be solved to their global optima by demonstrating their equivalence to convex optimization problems. 
\item By analyzing these equivalent convex programs, we provide interpretability to the optimization objectives of these attention modules.
\item Our experiments validate the (convex) vision transformers perform better than baseline convex methods in a transfer learning task. 
\end{itemize}


\subsection{Related Work}
This work is primarily related to two lines of research.

\noindent\textbf{Interpreting attention}. One approach has been to experimentally observe the properties of attention networks to understand them. For example, DINO proposes a contrastive self-supervised learning method for ViTs \cite{caron2021emerging}. It is observed that learned attention maps preserve semantic regions of the images. Another work compares the alignment across layers of trained ViTs and CNNs, concluding that ViTs have more a uniform representation structure across layers of the network \cite{raghu2021vision}. Another work uses a Deep Taylor Decomposition approach to visualize portions of input image leading to a particular ViT prediction \cite{chefer2021transformer}. 

Another approach is to analyze the expressivity of attention networks. One work interprets multi-head self-attention as a Bayesian inference, and provides tools to decide how many heads to use, and how to enforce distinct representations in different heads \cite{an2020repulsive}. Other analysis has demonstrated that sparse transformers can universally approximate any function \cite{yun2020n}, that multi-head self-attention is at least as expressive as a convolutional layer \cite{cordonnier2019relationship}, and that dot-product self-attention is not Lipschitz continuous \cite{kim2021lipschitz}.

\noindent\textbf{Convex neural networks}. Starting with \cite{pilanci2020neural}, there has been a long line of work demonstrating that various ReLU-activation neural network architectures have equivalent convex optimization problems. These include two-layer convolutional and vector-output networks \cite{ergen2020implicit, sahiner2020vector, sahiner2020convex}, as well as deeper ones \cite{ergen2021global}, networks with Batch Normalization \cite{ergen2021demystifying}, and Wasserstein GANs \cite{sahiner2021hidden}. Recent work has also demonstrated how to efficiently optimize the simplest forms of these equivalent convex networks, and incorporate additional constraints for adversarial robustness \cite{mishkin2022fast, bai2022efficient}. However, none of these works have analyzed the building blocks of transformers, which are a leading method in many state-of-the-art vision and language processing tasks. 

\section{Preliminaries}
In general, we analyze supervised learning problems where input training data \(\{\vec{X}_i \in \mathbb{R}^{s \times d}\}_{i=1}^n\) are the result of a patch embedding layer, and we have corresponding labels of arbitrary size \(\{\vec{Y}_i \in \mathbb{R}^{r \times c}\}_{i=1}^n\). For arbitrary convex loss function \(\mathcal{L}(\cdot, \cdot)\), we solve the optimization problem 
\begin{equation}
    p^* := \min_{\theta} \sum_{i=1}^n \mathcal{L}\left(f_{\theta}(\vec{X}_i), \vec{Y}_i \right) + \mathcal{R}(\theta)
\end{equation}

for some learnable parameters \(\theta\), neural network \(f_{\theta}(\cdot)\), and regularizer \(\mathcal{R}(\cdot)\). Note that this formulation encapsulates both denoising and classification scenarios: in the classification setting of \(r = 1\), one can absorb global average pooling into the convex loss \(\mathcal{L}\), whereas if \(r = s\), one can directly use squared loss or other convex loss functions. One may also use this formulation to apply to both supervised  and self-supervised learning.

In this paper, we denote \((\cdot)_+ := \max \{0, \cdot\}\) as the ReLU non-linearity. We use superscripts, say \(\vec{A}^{(i_i, i_2)}\), to denote blocks of matrices, and brackets, say \(\vec{A}[i_1, i_2]\), to denote elements of matrices, where the arguments refer to row (or block of rows) \(i_1\) and column (or block of columns) \(i_2\). 

\subsection{Implicit Convexity of Linear and ReLU MLPs}\label{sec:cvx_background}
Previously, it has been demonstrated that standard two-layer ReLU MLPs are equivalent to convex optimization problems. We briefly describe the relevant background to provide context for much of the analysis in this paper. In particular, we are presented with a network with \(m\) neurons in the hidden layer, weight-decay parameter \(\beta > 0\), and data \(\vec{X} \in \mathbb{R}^{n \times d}\), \(\vec{Y} \in \mathbb{R}^{n \times c}\):
\begin{align}
    p_{RMLP}^* := \min_{\vec{w}_{2j}, \vec{w}_{2j}} &\mathcal{L}\left(\sum_{j=1}^m (\vec{X}\vec{w}_{1j})_+\vec{w}_{2j}^\top, \vec{Y} \right) \nonumber \\
    &+ \frac{\beta}{2} \sum_{j=1}^m \left(\|\vec{w}_{1j}\|_2^2 + \|\vec{w}_{2j}\|_2^2\right). 
\end{align}

While this problem is non-convex as stated, it has been demonstrated that the objective is equivalent to the solution of an equivalent convex optimization problem, and there is a one-to-one mapping between the two problems' solutions \cite{sahiner2020vector}. In particular, this analysis makes use of \emph{hyperplane arrangements}, which enumerate all possible activation patterns of the ReLU non-linearity:
\begin{equation}
    \mathcal{D} := \{\mathrm{diag}\left(\mathbbm{1}\{\vec{X}\vec{u} \geq 0\}\right):\,\vec{u} \in \mathbb{R}^d\}
\end{equation}
The set \(\mathcal{D}\) is clearly finite, and its cardinality is bounded as \(P := |\mathcal{D}| \leq 2r\left(\frac{e(n-1)}{r}\right)^r\), where \(r := \mathrm{rank}(\data)\) \cite{stanley2004introduction, pilanci2020neural}. Through convex duality analysis, we can express an equivalent convex optimization problem by enumerating over the finite set of arrangements \(\{\vec{D}_j\}_{j=1}^P\). We define the following norm
\begin{align}\label{const_nuc_norm}
    &\|\vec{Z}\|_{*, \mathrm{K}} := \min_{t \geq 0} t \; \text{s.t.}\, \vec{Z} \in t\mathcal{C}\\
    &\mathcal{C} := \mathrm{conv}\{\vec{Z} = \vec{u}\vec{v}^\top: \; \vec{K}\vec{u} \geq 0, \, \|\vec{Z}\|_* \leq 1, \vec{v} \in \mathbb{R}^c\}. \nonumber
\end{align}
This norm is a quasi-nuclear norm, which differs from the standard nuclear norm in that the factorization upon which it relies puts a constraint on its left factors, which in our case will be an affine constraint. In convex ReLU neural networks, \(\mathrm{K}\) is chosen to enforce the existence of \(\{\vec{u}_k, \vec{v}_k\}\) such that \(\vec{Z} = \sum_k \vec{u}_k \vec{v}_k^\top\) and \(\vec{D}_{ j}\vec{X}\vec{Z} = \sum_k (\vec{X}\vec{u}_k)_+\vec{v}_k^\top\), and penalizes \(\sum_k \|\vec{u}_k\|_2 \|\vec{v}_k\|_2\) \footnote{We illustrate an example of \(\|\vec{Z}\|_{*, \mathrm{K}}\) in the Appendix \ref{sec:constrained_nuc}.}.\\\\
With this established, it can be shown that 
\begin{align}\label{eq:relu_mlp_convex}
    p_{RMLP}^* &= \min_{\{\vec{Z}_j\}_{j=1}^P} \mathcal{L}(\sum_{j=1}^P \vec{D}_j \vec{X}\vec{Z}_j, \vec{Y}) + \beta \sum_{j=1}^P \|\vec{Z}_j\|_{*, \mathrm{K}_j} \nonumber \\
    \vec{K}_j &:= (2\vec{D}_j - \vec{I})\data.
\end{align}
The two-layer ReLU MLP optimization problem is thus expressed as a piece-wise linear model with a constrained nuclear norm regularization. This contrasts with the two-layer linear activation MLP, whose convex equivalent is
\begin{equation}\label{eq:linear_mlp_convex}
    p^*_{LMLP} = \min_{\vec{Z}}\mathcal{L}\left(\data\vec{Z}, \vec{Y}\right) + \beta \|\vec{Z}\|_*. 
\end{equation}
This nuclear norm penalty is known to encourage low-rank solutions \cite{candes2010power, recht2010guaranteed} and appears in matrix factorization problems \cite{gunasekar2017implicit}. One can also define the \emph{gated ReLU} activation, where the ReLU gates are fixed to some \(\{\vec{h}_j\}_{j=1}^m\) \citep{fiat2019decoupling}
\begin{equation}
    g(\data \weight_{1j}) := \mathrm{diag}\left(\mathbbm{1}\{\vec{X}\vec{h}_j \geq 0\}\right) (\data \weight_{1j}).
\end{equation}
Then, defining \(\{\vec{D}_j\}_{j=1}^m := \{\mathrm{diag}\left(\mathbbm{1}\{\vec{X}\vec{h}_j \geq 0\} \right)\}_{j=1}^m\), the corresponding convex gated ReLU activation two-layer network objective follows directly from the ReLU and linear cases, and is given by 
\begin{equation}
    p_{GMLP}^* = \min_{\{\vec{Z}_j\}_{j=1}^m} \mathcal{L}\left(\sum_{j=1}^m \vec{D}_j \vec{X}\vec{Z}_j, \vec{Y} \right) + \beta \sum_{j=1}^m \|\vec{Z}_j\|_{*}.
\end{equation}
We note that for both linear and gated ReLU formulations, the regularization on the convex weights becomes the standard nuclear norm, since ReLU constraints no longer have to be enforced. It has been demonstrated that there is a small approximation gap between Gated ReLU and ReLU networks, and ReLU networks can be formed from the solutions to Gated ReLU problems \cite{mishkin2022fast}. \\\\
In terms of efficient algorithms to solve these problems, using an accelerated proximal gradient descent algorithm applied to the convex linear and gated ReLU programs, one can achieve an \(\epsilon\)-optimal solution in \(\mathcal{O}(1/\sqrt{\epsilon})\) iterations \cite{toh2010accelerated}. For the convex ReLU formulation,\cite{sahiner2020vector} proposes a Frank-Wolfe algorithm for convex MLPs which can be adapted to this case, which in the general case requires \(\mathcal{O}(1/\epsilon)\) iterations for \(\epsilon\)-optimality \cite{jaggi2013revisiting}.\\\\
In the subsequent sections, we will demonstrate how common vision transformer blocks with linear and ReLU activations can be related to equivalent convex optimization problems through similar convex duality techniques\footnote{Gated ReLU analysis is also provided in Appendix \ref{sec:grelu}}.

\section{Implicit Convexity of Self-Attention}
The canonical Vision Transformer (ViT) uses self-attention and MLPs as its backbone \cite{dosovitskiy2020image}. In particular, a single ``head" of a self-attention network is given by the following:
\begin{equation}
    f_j(\vec{X}_i) := \sigma\left(\frac{\vec{X}_i \vec{Q}_j \vec{K}_j^\top \vec{X}_i^\top}{\sqrt{d}}\right) \vec{X}_i \vec{V}_j,
\end{equation}
where \(\vec{Q}, \vec{K}, \vec{V}\) are all learnable parameters and \(\sigma(\cdot)\) typically (but not always) represents the softmax non-linearity. In practice, one typically uses \(m\) ``heads" of attention which are concatenated along the feature dimension, and this is followed by a ``channel-mixing" layer, or alternatively a classification head:
\begin{align}\label{eq:mhsa_original}
    f_{MHSA}(\vec{X}_i) &:= \begin{bmatrix} f_1(\vec{X}_i) & \cdots & f_m(\vec{X}_i) \end{bmatrix} \vec{W} \nonumber \\
    &= \sum_{j=1}^m \sigma\left(\frac{\vec{X}_i \vec{Q}_j \vec{K}_j^\top \vec{X}_i^\top}{\sqrt{d}}\right) \vec{X}_i \vec{V}_j\vec{W}_j
\end{align}
For the purpose of our analysis, noting that both \(\vec{Q}_j \vec{K}_j^\top\) and \(\vec{V}_j\vec{W}_j\) can be expressed by a single linear layer, we model the multi-head self-attention network as
\begin{align}
    f_{SA}(\vec{X}_i) := \sum_{j=1}^m \sigma\left(\frac{\vec{X}_i \weightmat_{1j} \vec{X}_i^\top}{\sqrt{d}}\right) \vec{X}_i \weightmat_{2j}. 
\end{align}
We then define the multi-head self-attention training problem as follows 
\begin{align}\label{eq:mhsa_opt_primal}
    p_{SA}^* := \min_{\weightmat_{1j}, \weightmat_{2j}} &\sum_{i=1}^n \mathcal{L}(f_{SA}(\vec{X}_i), \vec{Y}_i) \nonumber \\
    &+ \frac{\beta}{2}\sum_{j=1}^m \|\weightmat_{1j}\|_F^2 + \|\weightmat_{2j}\|_F^2. 
\end{align}
We thus employ any generic convex loss function and standard weight-decay in our formulation. While direct convex analysis when \(\sigma(\cdot)\) represents the softmax activation is intractable, we can analyze this architecture for many other activation functions. In particular, self-attention with both linear and ReLU activation functions has been proposed with performance on par to standard softmax activation networks \cite{shen2021efficient, yorsh2021simpletron, yorsh2021pureformer, zhang2021sparse}. Thus, we will analyze linear, ReLU, and gated ReLU activation variants of the multi-head self-attention. 

\begin{theorem}\label{theo:lin_attention}
For the linear activation multi-head self-attention training problem \eqref{eq:mhsa_opt_primal}, for \(m \geq m^*\) where \(m^* \leq \min\{d^2, dc\}\), the standard non-convex training objective is equivalent to a convex optimization problem, given by 
\begin{align}\label{eq:mhsa_linear_convex}
    p_{SA}^* = \min_{\vec{Z} \in \mathbb{R}^{d^2 \times dc} }  &\sum_{i=1}^n \mathcal{L}\left(\sum_{k=1}^d \sum_{\ell=1}^d \vec{G}_i[k, \ell] \data_i \vec{Z}^{(k, \ell)}, \labelmat_i\right) \nonumber\\&+\beta  \|\vec{Z}\|_*
\end{align}
where \(\vec{G}_i := \vec{X}_i^\top \vec{X}_i\) and \(\vec{Z}^{(k, \ell)} \in \mathbb{R}^{d \times c}\).
\end{theorem}
Our result demonstrates that a linear activation self-attention model consists of a Gram (feature correlation) matrix weighted linear model, with a nuclear norm penalty which groups the individual models to each other\footnote{Furthermore, there is a one-to-one mapping between the solutions of the convex and non-convex programs, which we describe in Appendix \ref{sec:proofs}.}.\\\\
One may also view the convex model as a set of linear models with a weighted nuclear norm, where each block \(\vec{Z}^{(k, \ell)}\) has a corresponding weight of \(1/\vec{G}_i[k, \ell]\). Thus, features with high correlation will have corresponding linear weights with larger norm. We note that when \(\beta = 0\), the linear self-attention model \eqref{eq:mhsa_linear_convex} is equivalent to the linear two-layer MLP \eqref{eq:linear_mlp_convex}.\\\\
While typically the nuclear norm penalty on \(\vec{Z}\) has no corresponding norm on each individual linear model \(\vec{Z}^{(k, \ell)}\), the following result summarizes an instance where the nuclear norm could decompose into smaller blocks.

\begin{corollary}\label{cor:lin_attention}
Assume some of the features of \(\vec{X}_i\) are entirely uncorrelated for all \(i\), i.e. \(\vec{G}_i\) is block diagonal with blocks \(\{\vec{G}_i^{(b)} \in \mathbb{R}^{d_b \times d_b}\}_{b=1}^B\) for all \(i\). Then, the convex program \eqref{eq:mhsa_linear_convex} reduces to the following convex program
\begin{align}\label{eq:mhsa_linear_convex_block_diag}
    p_{SA}^* = \min_{\vec{Z}^{(b)}}  &\sum_{i=1}^n \mathcal{L}\left(\sum_{b=1}^B \sum_{k=1}^{d_b} \sum_{\ell=1}^{d_b} \vec{G}_{i}^{(b)}[k, \ell] \data_i \vec{Z}^{(b, k, \ell)}, \labelmat_i\right) \nonumber \\ & + \beta  \sum_{b=1}^B \|\vec{Z}^{(b)}\|_*,\,\vec{Z}^{(b)} \in \mathbb{R}^{d_bd \times d_b c}.
\end{align}
\end{corollary}

This corollary thus demonstrates that under the assumption of sets of uncorrelated features, a linear self-attention block separates over these sets. In particular, the blocks of \(\vec{Z}\) corresponding to values of \(0\) in the Gram matrix \(\vec{G}_i\) will be set to \(0\), eliminating interactions between uncorrelated features. This phenomenon is illustrated in Figure \ref{fig:cvx-sa}.


While this linear model provides a simple, elegant explanation for the underpinnings of self-attention, we can also analyze self-attention blocks with non-linearities. We thus provide an analysis of ReLU-activation self-attention. 

\begin{theorem}\label{theo:relu_attention}
For the  ReLU activation multi-head self-attention training problem \eqref{eq:mhsa_opt_primal}, we define 
\begin{align*}
    \vec{X} &:= \begin{bmatrix} \vec{X}_1 \otimes \vec{X}_1 \\ \cdots \\ \vec{X}_n \otimes \vec{X}_n  \end{bmatrix} \\
    \{\vec{D}_j\}_{j=1}^P &:= \{\mathrm{diag}\left(\mathbbm{1}\{\vec{X}\vec{u}_j \geq 0\}\right):\;\vec{u}_j \in \mathbb{R}^{d^2}\},
\end{align*}


where \(P \leq 2r\left(\frac{e(n-1)}{r}\right)^r\) and \( r := \mathrm{rank}(\vec{X})\).
Then, for \(m \geq m^*\) where \(m^* \leq n\min\{d^2, dc\}\), the standard non-convex training objective is equivalent to a convex optimization problem, given by
\begin{alignat}{9}\label{eq:mhsa_relu_convex}
    p_{SA}^* &= \min_{\vec{Z}_j \in \mathbb{R}^{d^2 \times dc} } \sum_{i=1}^n \mathcal{L}\left(\sum_{j=1}^P \sum_{\ell=1}^d \sum_{k=1}^d \vec{G}_{i,j}^{(k, \ell)}\vec{X}_i\vec{Z}_j^{(k, \ell)}, \labelmat_i\right) \nonumber \\
    &+ \beta \sum_{j=1}^m \|\vec{Z}_j\|_{*, \mathrm{K}_{j}},\, \vec{K}_{j} := (2\vec{D}_j - \vec{I}_{ns^2})\vec{X},
\end{alignat}
where \(\vec{G}_{i,j} := (\vec{X}_i \otimes \vec{I}_{s})^\top \vec{D}_{j}^{(i)} (\vec{X}_i \otimes \vec{I}_{s})\), \(\vec{G}_{i,j}^{(k, \ell)} \in \mathbb{R}^{s \times s}\) and \(\vec{Z}_j^{(k, \ell)} \in \mathbb{R}^{d \times c}\).
\end{theorem}



Interestingly, while the hyperplane arrangements for a standard ReLU MLP depend only on the data matrix \(\data\), for a self-attention network they are more complex, instead depending on \(\data_i \otimes \data_i\). These hyperplane arrangements define the constraints for the constrained nuclear norm penalty. One could potentially see a ReLU activation self-attention model as a fusion of two models--one which uses \(\vec{X}_i \otimes \vec{X}_i\) for generating hyperplane arrangements, and one which uses \(\vec{X}_i\) for linear predictions. Thus, unlike the linear self-attention case, even in the case of \(\beta = 0\), the ReLU self-attention network \eqref{eq:mhsa_relu_convex} is not equivalent to the ReLU MLP model \eqref{eq:relu_mlp_convex}.  \\\\
Furthermore, while in the linear-activation case in \eqref{eq:mhsa_linear_convex}, each linear model was scaled by a single entry in \(\vec{G}_i\), in the ReLU case each linear model is scaled by a diagonal matrix \(\vec{G}_{i,j}^{(k, \ell)}\) which combines second-order information from \(\vec{X}_i\) with the hyperplane arrangements induced by the ReLU activation function. One may, for example, note the identity
\begin{equation}
    \vec{G}_{i,j}^{(k, \ell)} = \sum_{t=1}^{s} \vec{X}_i[t, k]\vec{X}_i[t, \ell] \vec{D}_{j}^{(i, t)},
\end{equation}
for diagonal \(\vec{D}_{j}^{(i, t)} \in \{0, 1\}^{s \times s}\). Therefore, \(\vec{G}_{i,j}^{(k, \ell)}\) can be viewed as a correlation between features \(k\) and \(\ell\) weighted by diagonal \(\{0, 1\}\)-valued hyperplane arrangements for each corresponding row \(t\), in other words a type of ``local" correlation, where locality is achieved by the \(\{0, 1\}\) values in \(\vec{D}_{j}^{(i, t)}\). This local correlation scales each token of the prediction, essentially giving weight to tokens which have been not been masked away by \(\vec{D}_{j}^{(i, t)}\). 

\section{Alternative Mixing Mechanisms}
While self-attention is the original proposed token mixer used for vision transformers, there are many other alternative approaches which have shown to produce similar results while being more computationally efficient. We tackle two such architectures here.
\subsection{MLP Mixer}
We begin by analyzing the MLP-Mixer architecture, an all-MLP alternative to self-attention networks with competitive performance on image classification benchmarks \cite{tolstikhin2021mlp}. The proposal is simple--apply an MLP along one dimension of the input, followed by an MLP along the opposite dimension. The simplest form of this MLP-Mixer architecture can thus be written as
\begin{align}\label{eq:mlp_mixer_primal}
    p_{MM}^* &:= \min_{\weightmat_{1j}, \weightmat_{2j}} \sum_{i=1}^n \mathcal{L}(\sum_{j=1}^m \sigma(\vec{W}_{1j} \vec{X}_i) \vec{W}_{2j}, \vec{Y}_i) \nonumber \\
    &+ \frac{\beta}{2}\sum_{j=1}^m \|\weightmat_{1j}\|_F^2 + \|\weightmat_{2j}\|_F^2, 
\end{align}
where \(\sigma\) is an activation function. While \cite{tolstikhin2021mlp} use the GeLU non-linearity \cite{hendrycks2016gaussian}, we analyze the simpler linear and ReLU activation counterparts, which shed important insights into the underlying structure of the MLP-Mixer architecture.

\begin{theorem}\label{theo:linear_mlpmixer}
For the linear activation MLP-Mixer training problem \eqref{eq:mlp_mixer_primal}, for \(m \geq m^*\) where \(m^* \leq \min\{s^2, dc\}\), the standard non-convex training objective is equivalent to a convex optimization problem, given by 
\begin{align}\label{eq:mlp_mixer_linear_convex}
    p_{MM}^* &= \min_{\vec{Z} \in \mathbb{R}^{s^2\times dc} }  \sum_{i=1}^n \mathcal{L}\left(\begin{bmatrix} f_1(\vec{X}_i) & \cdots &  f_c(\vec{X}_i)  \end{bmatrix}, \labelmat_i\right)\nonumber\\
    &+ \beta \|\vec{Z}\|_{*},\,f_p(\vec{X}_i) := \vec{Z}^{(p)}\vecop(\data_i)
\end{align}
where \(\vec{Z}^{(p)} \in \mathbb{R}^{s \times sd}\) for \(p \in [c]\), for  \(\vec{Z}^{(p, t)} \in \mathbb{R}^{s \times d}\) for \(t \in [s]\), and 
\[\vec{Z}^{(p)} = \begin{bmatrix}\vec{Z}^{(p, 1)}  & \cdots & \vec{Z}^{(p, s)} \end{bmatrix}, \vec{Z} = \begin{bmatrix} \vec{Z}^{(1, 1)} & \cdots & \vec{Z}^{(c, 1)}  \\ & \cdots & \\  \vec{Z}^{(1, s)} & \cdots & \vec{Z}^{(c, s)} \end{bmatrix}. \]
\end{theorem}
We can contrast the fitting term of the linear MLP-Mixer to a standard linear MLP \eqref{eq:linear_mlp_convex}, where each column \(k\) of the network output is given by \[\vec{X}_i \vec{Z}^{(k)} = ({\vec{Z}^{(k)}}^\top \otimes \vec{I}_s)\vecop(\data_i),\] where \(\vec{Z}^{(k)} \in \mathbb{R}^d\). Thus, the MLP-Mixer gives the network \(s^2\) more degrees of freedom for fitting each column of \(\vec{Y}_i\) than a standard linear MLP. This demonstrates that unlike the linear self-attention network, a linear MLP-Mixer model is not equivalent to a linear standard MLP even when \(\beta = 0\). One may speculate that this additional implicit degrees of freedom allows mixer-like MLP models to fit complex distributions more easily compared to standard MLPs. While it appears from the fitting term that each output class of \(\vec{Y}_i\) is fit independently, we note that these outputs are coupled together by the nuclear norm on \(\vec{Z}\), which encourages \{ \(\vec{Z}^{(k)}\}_{k=1}^c\) to be similar to one another. \\\\
Another interpretation of the convexified linear MLP-Mixer architecture may be achieved by simply permuting the columns of \(\vec{Z}\) to form \(\tilde{\vec{Z}}\), which does not affect the nuclear norm and thus does not impact the optimal solution. If one partitions the columns of \(\tilde{\vec{Z}}\) according to blocks \(\tilde{\vec{Z}}^{(t, k)} \in \mathbb{R}^{s \times c}\) for \(t \in [s]\), \(k \in [d]\), one can also write 
\begin{align}\label{eq:mlp_mixer_linear_convex_permuted}
    p_{MM}^* &= \min_{\tilde{\vec{Z}} }  \sum_{i=1}^n \mathcal{L}\left( \sum_{t=1}^s \sum_{k=1}^d \data_i[t, k]\tilde{\vec{Z}}^{(t, k)} , \labelmat_i\right) + \beta \|\tilde{\vec{Z}}\|_{*}
\end{align}
Here, the connections to the linear self-attention network \eqref{eq:mhsa_linear_convex} become more clear. While \eqref{eq:mhsa_linear_convex} is a weighted summation of linear models, where the weights correspond to Gram matrix entries, \eqref{eq:mlp_mixer_linear_convex_permuted} is a weighted summation of predictions, where the weights correspond to data matrix entries. We also note that in most networks, typically \(s < d\), so the MLP-Mixer block has a lower order complexity to solve compared to the self-attention block. We can also extend these results to ReLU activation MLP-Mixers.
\begin{theorem}\label{theo:relu_mlpmixer}
For the ReLU activation MLP-Mixer training problem \eqref{eq:mlp_mixer_primal}, we define 
\begin{align*}
    \vec{X} &:= \begin{bmatrix} \vec{X}_1^\top  \otimes \vec{I}_{s} \\ \cdots \\ \vec{X}_n^\top  \otimes \vec{I}_{s}  \end{bmatrix}\\
    \{\vec{D}_j\}_{j=1}^P &:= \{\mathrm{diag}\left(\mathbbm{1}\{\vec{X}\vec{u}_j \geq 0\}\right):\;\vec{u}_j \in \mathbb{R}^{s^2}\},
\end{align*}
where \(P \leq 2r\left(\frac{e(n-1)}{r}\right)^r\) and \( r := \mathrm{rank}(\vec{X})\).
Then, for \(m \geq m^*\) where \(m^* \leq n\min\{s^2, dc\}\), the standard non-convex training objective is equivalent to a convex optimization problem, given by 
\begin{align}\label{eq:mlp_mixer_relu_convex}
    p_{MM}^* &= \min_{\vec{Z}_j \in \mathbb{R}^{s^2\times dc} }  \sum_{i=1}^n \mathcal{L}\left(\begin{bmatrix} f_1(\vec{X}_i) & \cdots &  f_c(\vec{X}_i) \end{bmatrix} , \labelmat_i\right) \nonumber \\
    &+ \beta \sum_{j=1}^P \|\vec{Z}_j\|_{*,\mathrm{K}_j}, \vec{K}_j := (2\vec{D}_j - \vec{I}_{nd})\vec{X}
\end{align}
where 
\begin{align*}
    f_p(\vec{X_i}) &:= \sum_{j=1}^P \begin{bmatrix}  \vec{D}_{j}^{(i, 1)}\vec{Z}_j^{(p, 1)} \cdots   \vec{D}_{j}^{(i, d)}\vec{Z}_j^{(p, d)}\end{bmatrix}  \vecop(\data_i)
\end{align*}
for \(\vec{D}_{j}^{(i,k)} \in \mathbb{R}^{s \times s}\) and \(\vec{Z}_j^{(p, k)} \in \mathbb{R}^{s \times s}\).
\end{theorem}
Now, unlike the self-attention model, where the effective data matrix for hyperplane arrangements was \(\data_i \otimes \data_i\), the MLP-mixer's arrangements use \(\data_i^\top \otimes \vec{I}_s\), providing additional degrees of freedom for partitioning the data while still incorporating only first-order information about the data. Using the same column permutation trick as in \eqref{eq:mlp_mixer_linear_convex_permuted}, one may write \eqref{eq:mlp_mixer_relu_convex} as 
\begin{align}\label{eq:mlp_mixer_relu_convex_permuted}
    p_{MM}^* &= \min_{\tilde{\vec{Z}}_j }  \sum_{i=1}^n \mathcal{L}\left( \sum_{j=1}^P \sum_{t=1}^s \sum_{k=1}^d \vec{X}_i[t,k]\vec{D}_j^{(i, k)}\tilde{\vec{Z}}_j^{(t, k)}, \labelmat_i\right) \nonumber \\
    &+ \beta \sum_{j=1}^P \|\tilde{\vec{Z}}_j\|_{*,\mathrm{K}_j},
\end{align}
where again now we see clearly the differences with ReLU self-attention \eqref{eq:mhsa_relu_convex}, with the diagonal arrangements weighted by \(\data_i\) rather than the Gram matrix, and instead of \(\tilde{\vec{Z}}_j^{(t, k)}\) being weights of a linear model, they are simply predictions. 

\subsection{Fourier Neural Operator}
In contrast to self-attention or MLP-like attention mechanisms, there is also a family of Fourier-based alternatives to self-attention which have recently shown promise in vision. We present Fourier Neural Operator (FNO) \cite{li2020fourier, guibas2021adaptive}, which works as follows: $i)$ 2D DFT is applied first over the spatial tokens; $ii)$ each token is multiplied by its own weight matrix; and $iii)$ inverse DFT returns the Fourier tokens back to the original (spatial) domain. 

To express FNO in a compact matrix form, note that in addition to standard MLP weights \(\weightmat_1 \in \mathbb{R}^{d \times m},\, \weightmat_2 \in \mathbb{R}^{m \times c}\), FNO blocks has a third set of weights \(\vec{L} \in \mathbb{R}^{s \times d \times d}\). Let us define the Fourier transform \(\vec{F} :=  \vec{F}_h \otimes \vec{F}_w\), that is a vectorized version of $h \times w$ 2D Fourier transform. It is more convenient to work with the weights $\vec{L}$ in the Fourier space as $\vec{V}$, where the 2D Fourier transform has been applied on the first dimension, namely $\vec{V}[:,i,j]$ for every $i,j$.

Now, define  $\vec{V}^{(j)}=\vec{V}[j,:,:]$.  Each row of \(\vec{F}\vec{X}_i\) is then multiplied by $d \times d$ weight matrix $\vec{V}^{(j)}$, and converted back to the image domain as follows
%
\begin{align}\label{eq:fno_full_problem_def}
    f_{FN}(\vec{X}_i) := \sigma\left( \left(\vec{F}^{-1}
    \begin{bmatrix}(\vec{F}\vec{X}_i)_1^\top \vec{V}^{(1)} \\  \cdots \\ (\vec{F}\vec{X}_i)_{s}^\top \vec{V}^{(s)} \end{bmatrix} \right)\weightmat_1 \right)\weightmat_2. 
\end{align}
%


This representation can be heavily simplified. 
\begin{lemma}\label{lem:fno}
For weights \(\weightmat_1 \in \mathbb{R}^{sd \times m}\), \(\weightmat_2 \in \mathbb{R}^{m \times c}\), the FNO block \eqref{eq:fno_full_problem_def} can be equivalently represented as
\begin{align}\label{eq:fno_simplified}
    f_{FN}(\vec{X}_i) = \sum_{j=1}^m \sigma\left(\mathrm{circ}(\vec{X}_i) \weight_{1j}\right)\weight_{2j}^\top
\end{align}
where \(\mathrm{circ}(\vec{X}_i) \in \mathbb{R}^{s \times sd}\) denotes a matrix composed of all \(s\) circulant shifts of \(\vec{X}_i\) along its first dimension.
\end{lemma}


We can accordingly write the FNO training objective as 
\begin{align} \label{eq:fno_primal}
p_{FN}^* := \min_{\weight_{1j}, \weight_{2j}} &\sum_{i=1}^n \mathcal{L}(f_{FN}(\vec{X}_i), \vec{Y}_i) \nonumber \\
    &+ \frac{\beta}{2}\sum_{j=1}^m \|\weight_{1j}\|_2^2 + \|\weight_{2j}\|_2^2. 
\end{align}
We note that FNO actually starkly resembles a two-layer CNN, where the first layer consists of a convolutional layer with full circulant padding and a \emph{global} convolutional kernel. Unlike a typical CNN, in which the kernel size is usually small and the convolution is \emph{local}, the convolution here is much larger, which means there are many more parameters than a typical CNN. Similar CNN architectures have previously been analyzed via convex duality \cite{ergen2020implicit, sahiner2020convex}. Accordingly, for both linear and ReLU activation, \eqref{eq:fno_primal} is equivalent to a convex optimization problem. 
\begin{theorem}\label{theo:linear_fno}
For the linear activation FNO training problem \eqref{eq:fno_primal}, for \(m \geq m^*\) where \(m^* \leq \min\{sd, c\}\), the standard non-convex training objective is equivalent to a convex optimization problem, given by 
\begin{align}\label{eq:fno_linear_convex}
    p_{FN}^* &= \min_{\vec{Z} \in \mathbb{R}^{sd \times c} }  \sum_{i=1}^n \mathcal{L}\left(\mathrm{circ}(\vec{X}_i) \vec{Z}, \labelmat_i\right) + \beta \|\vec{Z}\|_{*}. 
\end{align}
\end{theorem}

\begin{theorem}\label{theo:relu_fno}
For the ReLU activation FNO training problem \eqref{eq:fno_primal}, we define 
\begin{align*}
    \vec{X} &:= \begin{bmatrix} \mathrm{circ}(\vec{X}_1)\\ \cdots \\ \mathrm{circ}(\vec{X}_n) \end{bmatrix}\\
    \{\vec{D}_j\}_{j=1}^P &:= \{\mathrm{diag}\left(\mathbbm{1}\{\vec{X}\vec{u}_j \geq 0\}\right):\;\vec{u}_j \in \mathbb{R}^{sd}\},
\end{align*}
where \(P \leq 2r\left(\frac{e(n-1)}{r}\right)^r\) and \( r := \mathrm{rank}(\vec{X})\). Then, for \(m \geq m^*\) where \(m^* \leq n\min\{sd, c\}\), the standard non-convex training objective is equivalent to a convex optimization problem, given by 
\begin{align}\label{eq:fno_relu_convex}
    p_{FN}^* &= \min_{\vec{Z}_j \in \mathbb{R}^{sd \times c} }  \sum_{i=1}^n \mathcal{L}\left(\sum_{j=1}^P \vec{D}_j^{(i)}\mathrm{circ}(\vec{X}_i) \vec{Z}_j, \labelmat_i\right) \nonumber \\
    &+ \beta \sum_{j=1}^P \|\vec{Z}_j\|_{*, \mathrm{K}_j},\,
    \vec{K}_j := (2\vec{D}_j - \vec{I}_{ns})\vec{X}.
\end{align}
\end{theorem}

\subsubsection{Block Diagonal FNO}
While the FNO formulation is quite elegant, it requires many parameters (\(d^2\) for each token). Accordingly, modifications have been proposed in the form of Adaptive Fourier Neural Operator (AFNO) \cite{guibas2021adaptive}. One important modification pertains to enforcing the token weights to obey a block diagonal structure. This has been significantly improved the training and generalization ability of AFNO compared with the standard FNO \footnote{A standard AFNO network also includes additional steps, including a soft-thresholding operator, which for simplicity we do not analyze here.}. We call this architecture B-FNO, which boils down to
\begin{align}\label{eq:bfno_model}
    f_{BFN}(\vec{X}_i) &:= \sigma\left( \vec{F}^{-1}
    \begin{bmatrix}(\vec{F}\vec{X}_i)_1^\top \vec{V}^{(1)} \\  \cdots \\ (\vec{F}\vec{X}_i)_{s}^\top \vec{V}^{(s)} \end{bmatrix}  \right)\weightmat_2 \\
    \vec{L}^{(l)} &:= \begin{bmatrix}\vec{L}^{(l, 1)} & & \\
    & \ddots & \\ &&\vec{L}^{(l, B)}\end{bmatrix} \in \mathbb{R}^{d \times m},\, l\in[s] \nonumber \\
     \weightmat_2 &:= \begin{bmatrix}\weightmat_2^{(1)} & & \\
    & \ddots & \\ &&\weightmat_2^{(B)}\end{bmatrix} \in \mathbb{R}^{m \times c},\nonumber
\end{align}
for \(B\) blocks. This can be simplified as
\begin{lemma}\label{lem:bfno}
For weights \(\weightmat_{1b} \in \mathbb{R}^{sd/B \times m/B}\) and \(\weightmat_{2b} \in \mathbb{R}^{m/B \times c/B}\), assuming \(\sigma\) operates element-wise, the B-FNO model \eqref{eq:bfno_model} can be equivalently represented as
\begin{align}\label{eq:bfno_fn}
    f_{BFN}(\vec{X}_i) &= \begin{bmatrix} f_{BFN}^{(1)}(\vec{X}_i) & \cdots & f_{BFN}^{(B)}(\vec{X}_i) \end{bmatrix}\\
    f_{BFN}^{(b)}(\vec{X}_i) &= \sum_{j=1}^m \sigma\left(\mathrm{circ}(\vec{X}_i^{(b)})  \weight_{1bj}\right)\weight_{2bj}^\top \nonumber
\end{align}
where \(\mathrm{circ}(\vec{X}_i^{(b)}) \in \mathbb{R}^{s \times sd/B}\) is a matrix composed of all \(s\) circulant shifts of \(\vec{X}_i^{(b)} \in \mathbb{R}^{s \times d/B}\) along its first dimension.
\end{lemma}
Interestingly, the block-diagonal weights of AFNO contrast the {\it local} convolution in CNNs with a {\it global} and \emph{group-wise} convolution with \(B\) groups. We thus define
\begin{align} \label{eq:bfno_primal}
p_{BFN}^* &:= \min_{\weightmat_{1bj}, \weightmat_{2bj}} \sum_{i=1}^n \mathcal{L}(f_{BFN}(\vec{X}_i), \vec{Y}_i) \nonumber \\
    &+ \frac{\beta}{2}\sum_{b=1}^B \sum_{j=1}^m\|\weight_{1bj}\|_2^2 + \|\weight_{2bj}\|_2^2. 
\end{align}

\begin{theorem}\label{theo:linear_bfno}
For the linear activation B-FNO training problem \eqref{eq:bfno_primal}, for \(m \geq m^*\) where \(m^* \leq 1/B\min\{sd, c\}\), the standard non-convex training objective is equivalent to a convex optimization problem, given by 
\begin{align}\label{eq:bfno_linear_convex}
    p_{BFN}^* &= \min_{\vec{Z}_b}  \sum_{i=1}^n \mathcal{L}\left(\begin{bmatrix} f^{(1)}(\vec{X}_i) & \cdots &  f^{(B)}(\vec{X}_i)  \end{bmatrix}, \labelmat_i\right), \nonumber \\
    &+ \beta \sum_{b=1}^B \|\vec{Z}_b\|_{*} \\
    \vec{Z}_b &\in \mathbb{R}^{sd/B \times c/B},\, f^{(b)} := \mathrm{circ}(\vec{X}_i^{(b)})\vec{Z}_b. 
    \nonumber 
\end{align}
\end{theorem}

\begin{theorem}\label{theo:relu_bfno}
For the ReLU activation B-FNO training problem \eqref{eq:bfno_primal}, we define 
\begin{align*}
    \vec{X}_b &:= \begin{bmatrix} \mathrm{circ}(\vec{X}_1^{(b)})\\ \cdots \\ \mathrm{circ}(\vec{X}_n^{(b)}) \end{bmatrix} \\
    \{\vec{D}_{b,j}\}_{j=1}^{P_b} &:= \{\mathrm{diag}\left(\mathbbm{1}\{\vec{X}_b\vec{u}_j \geq 0\}\right):\;\vec{u}_j \in \mathbb{R}^{sd/B}\},
\end{align*}
where \(P_b \leq 2r_b\left(\frac{e(n-1)}{r_b}\right)^{r_b}\) and \( r_b := \mathrm{rank}(\vec{X}_b)\). Then, for \(m \geq m^*\) where \(m^* \leq n/B\min\{sd, c\}\), the standard non-convex training objective is equivalent to a convex optimization problem, given by 
\begin{align}\label{eq:bfno_relu_convex}
    p_{BFN}^* &= \min_{\vec{Z}_{b, j} }  \sum_{i=1}^n \mathcal{L}\left(\begin{bmatrix} f^{(1)}(\vec{X}_i) & \cdots &  f^{(B)}(\vec{X}_i)  \end{bmatrix}, \labelmat_i\right) \nonumber \\
    &+ \beta \sum_{b=1}^B \sum_{j=1}^{P_b} \|\vec{Z}_{j, b}\|_{*, \mathrm{K}_{b,j}},
\end{align}
where
\begin{align*}
    \vec{Z}_{b,j} &\in \mathbb{R}^{sd/B \times c/B}, \, f^{(b)}(\data_i) := \sum_{j=1}^{P_b} \vec{D}_{b,j}\mathrm{circ}(\vec{X}_i^{(b)})\vec{Z}_{b,j} \\
    \vec{K}_{b,j} &:= (2\vec{D}_{b,j} - \vec{I}_{ns})\vec{X}_b.
\end{align*}
\end{theorem}

\section{Numerical Results}
In this section, we seek to compare the performance of the transformer heads we have analyzed in this work to baseline convex optimization methods. This comparison allows us to illustrate the implicit biases imposed by these novel heads in a practical example. In particular, we consider the task of training a single new block of these convex heads for performing an image classification task. This is essentially the task of transfer learning {\it without fine-tuning existing weights of the backbone network}, which may be essential in computation and memory-constrained settings at the edge. For few-shot fine-tuning transformer tasks, non-convex optimization has observed to be unstable under different random initializations \cite{mosbach2020stability}. Furthermore, fine-tuning only the final layer a network is a common practice, which performs very well in spurious correlation benchmarks \cite{kirichenko2022last}. \\\\
Specifically, we seek to classify images from the CIFAR-100 dataset \cite{krizhevsky2009learning}. We first generate embeddings from a pretrained gMLP-S model \cite{liu2021pay} on \(224 \times 224\) images from the ImageNet-1k dataset \cite{deng2009imagenet} with \(16 \times 16\) patches (\(s=196\), \(d=256\)). We then finetune the \emph{single convex head} to classify images from CIFAR-100, while leaving the pre-trained backbone fixed. \\\\
For the backbone gMLP architecture, we reduce the feature dimension to \(d=100\) with average pooling as a pre-processing step before training with the convex heads.  Similarly, for computational efficiency, we train the Gated ReLU variants of the standard ReLU architectures, since these Gated ReLU activation networks are unconstrained. For BFNO, we choose \(B=5\). All the heads use identical dimensions \((d=100, s=196, c=100)\), and we choose the number of neurons in ReLU heads to be \(m=100\), except in the case of self-attention, where we choose \(m=5\) to have the parameter count be roughly equal across heads (see Table \ref{tab:params} in Appendix \ref{sec:appendix}). As our baseline, we compare to a simple linear model (i.e. logistic regression) and to convex equivalents of MLPs as discussed in Section \ref{sec:cvx_background}. \\\\
We summarize the results in Table \ref{tab:results}. Here, we demonstrate that the attention variants outperform the standard convex MLP and linear baselines. This suggests that the higher-order information and additional degrees of freedom of the attention architectures provide advantages for difficult vision tasks. Surprisingly, for self-attention, FNO, and MLP-Mixer, there is only a marginal gap between the linear and ReLU activation performance, suggesting most of the benefit of these architectures is in their fundamental structure, rather than the nonlinearity which is applied. In contrast, for B-FNO, there is a very large gap between ReLU and linear activation accuracies, suggesting that this nonlinearity is more crucial when group convolutions are applied. These convexified architectures thus pave the way towards stable and transparent models for transfer learning. 
\begin{table}[t]
\caption{CIFAR-100 classification accuracy for training a single \emph{convex} head. Embeddings are generated from gMLP-S pre-trained on ImageNet. Note that the backbone is not fine-tuned.}
\label{tab:results}
\begin{center}
\begin{small}
\begin{sc}
\begin{tabular}{lcccc}
\toprule
Convex Head & Act. & Top-1 & Top-5\\
\midrule
Self-Attention & \multirow{6}*{Linear} & 73.81 & 92.87\\
MLP-Mixer & & \bf{78.11} & \bf{94.79} \\
B-FNO & & 68.68 & 90.62\\
FNO & & 72.29 & 93.03  \\
MLP & & 65.95 & 89.33 \\
Linear & & 66.42 & 89.27\\
\midrule
Self-Attention & \multirow{5}*{ReLU} & 74.74 & 93.45\\
MLP-Mixer & & \bf{80.22} & \bf{95.79}\\
B-FNO & & 77.65 & 94.97\\
FNO & & 72.93 & 92.71 \\
MLP & & 73.05 & 92.52 \\ 
\bottomrule
\end{tabular}
\end{sc}
\end{small}
\end{center}
\vskip -0.1in
\end{table}

\section{Conclusion}
We demonstrated that blocks of self-attention and common alternatives such as MLP-Mixer, FNO, and B-FNO are equivalent to convex optimization problems in the case of linear and ReLU activations. These equivalent convex formulations implicitly cluster correlated features, and are penalized with a block nuclear norm regularizer which ensures a global representation. For future work, it remains to craft efficient approximate solvers of these networks by leveraging the structure of these unique regularizers. Faster solvers may implemented, such as FISTA or related algorithms (some of which were explored in the context of convex MLPs \cite{mishkin2022fast}). In the long term for practical adoption, future theoretical work would also require analysis of deeper networks as are often used in practice. One may use this work in designing new network architectures by specifying the desired convex formulation. 

\subsubsection*{Acknowledgments}
This work was partially supported by the National Science Foundation under grants ECCS-2037304, DMS-2134248, the Army Research Office. Research reported in this work was supported by the National Institutes of Health under award numbers R01EB009690 and U01EB029427. The content is solely the responsibility of the authors and does not necessarily represent the official views of the National Institutes of Health. 
\newpage

\bibliography{icml2022}
\bibliographystyle{icml2022}

\clearpage
\onecolumn
\appendix
\addcontentsline{toc}{section}{Appendix} 
\part{Appendix} 
\parttoc 

\section{Proofs}\label{sec:proofs}

\subsection{Preliminaries}
Here, we describe the general technique which allows for convex duality of self-attention and similar architectures. In particular, all of the proofs for the theorems in the subsequent sections follow the following general form:
\begin{enumerate}
    \item Re-scale the weights such that the optimization problem has a Frobenius norm penalty on the second layer weights with a norm constraint on the first layer weights. 
    \item Form the dual problem over the second-layer weights, creating a dual constraint which depends on the first layer weights. 
    \item Solve the dual constraint over the first layer weights.
    \item Form the Lagrangian problem and solve over the dual weights. 
    \item Make any simplifications as necessary. 
\end{enumerate}
We describe the first two steps here, and use it in the following proofs. 
\begin{lemma}\label{lem:rescaling}
Suppose we are given an optimization problem of the form
\begin{equation}
    p^* := \min_{\weightmat_{1j}, \weightmat_{2j}} \sum_{i=1}^n \mathcal{L}(\sum_{j=1}^m g(\data_i; \weightmat_{1j}) \weightmat_{2j}, \labelmat_i) + \frac{\beta}{2} \sum_{j=1}^m \|\weightmat_{1j}\|_F^2 +  \|\weightmat_{2j}\|_F^2,
\end{equation}
where \(g(\data_i; \weightmat_{1j})\) is any function such that \(g(\data_i; \alpha_j\weightmat_{1j}) = \alpha_jg(\data_i; \weightmat_{1j})\) for \(\alpha_j >0\). Then, this problem is equivalent to 
\begin{equation}
    p^* = \min_{\|\weightmat_{1j}\|_F \leq 1, \weightmat_{2j}} \sum_{i=1}^n \mathcal{L}(\sum_{j=1}^m g(\data_i; \weightmat_{1j}) \weightmat_{2j}, \labelmat_i) + \beta \sum_{j=1}^m  \|\weightmat_{2j}\|_F.
\end{equation}
\end{lemma}
\begin{proof}
Due to the form of \(g(\data_i; \weightmat_{1j})\), we can rescale the parameters as \(\bar{\weightmat}_{1j} = \alpha_j \bar{\weightmat}_{1j}\), \(\bar{\weightmat}_{2j} = \bar{\weightmat}_{2j}/\alpha_j\) for \(\alpha_j > 0\) without changing the network output. Then, to minimize the regularization term, we can write the problem as 
\begin{equation}
    p^* := \min_{\weightmat_{1j}, \weightmat_{2j}} \min_{\alpha_j> 0} \sum_{i=1}^n \mathcal{L}(\sum_{j=1}^m g(\data_i; \weightmat_{1j}) \weightmat_{2j}, \labelmat_i) + \frac{\beta}{2} \sum_{j=1}^m \alpha_j^2 \|\weightmat_{1j}\|_F^2 +  \|\weightmat_{2j}\|_F^2/\alpha_j^2.
\end{equation}
Solving this minimization problem over \(\alpha_j\) \cite{savarese2019infinite,sahiner2020vector}, we obtain 
\begin{equation}
    p^* := \min_{\weightmat_{1j}, \weightmat_{2j}} \sum_{i=1}^n \mathcal{L}(\sum_{j=1}^m g(\data_i; \weightmat_{1j}) \weightmat_{2j}, \labelmat_i) + \beta \sum_{j=1}^m \|\weightmat_{1j}\|_F \|\weightmat_{2j}\|_F.
\end{equation}
We can thus set \(\|\weightmat_{1j}\|_F = 1\) without loss of generality, and further relaxing this to \(\|\weightmat_{1j}\|_F \leq 1\) does not change the optimal solution. Thus, we are left with the desired result:
\begin{equation}
    p^* = \min_{\|\weightmat_{1j}\|_F \leq 1, \weightmat_{2j}} \sum_{i=1}^n \mathcal{L}(\sum_{j=1}^m g(\data_i; \weightmat_{1j}) \weightmat_{2j}, \labelmat_i) + \beta \sum_{j=1}^m  \|\weightmat_{2j}\|_F.
\end{equation}
\end{proof}
Note that the assumption on \(g\) encapsulates all architectures studied in this work: self-attention, MLP-Mixer, FNO, B-FNO, and other extensions with ReLU, gated ReLU, or linear activation functions all satisfy this property. 
\begin{lemma}\label{lem:duality}
Suppose that
\begin{equation}
    p^* := \min_{\|\weightmat_{1j}\|_F \leq 1, \weightmat_{2j}} \sum_{i=1}^n \mathcal{L}(\sum_{j=1}^m g(\data_i; \weightmat_{1j}) \weightmat_{2j}, \labelmat_i) + \beta \sum_{j=1}^m  \|\weightmat_{2j}\|_F.
\end{equation}
Then, for all \(\beta > 0\), if \(m \geq m^*\) for some \(m^*\), this optimization problem is equivalent to 
\begin{align}
    p^* &= \max_{\dualmat_i} -\sum_{i=1}^n \mathcal{L}^*(\dualmat_i, \labelmat_i) \nonumber \\
    \text{s.t.}~&\max_{\|\weightmat_1\|_F \leq 1} \|\sum_{i=1}^n \dualmat_i^\top g(\data_i; \weightmat_{1})\|_F \leq \beta,
\end{align}
where \(\mathcal{L}^*\) is the Fenchel conjugate of \(\mathcal{L}\). 
\end{lemma}
\begin{proof}
We first can re-write the problem as
\begin{equation}
    p^* = \min_{\|\weightmat_{1j}\|_F \leq 1} \min_{ \weightmat_{2j}, \vec{R}_i} \sum_{i=1}^n \mathcal{L}(\vec{R}_i, \labelmat_i) + \beta \sum_{j=1}^m  \|\weightmat_{2j}\|_F~\text{s.t.}~\vec{R}_i = \sum_{j=1}^m g(\data_i; \weightmat_{1j}) \weightmat_{2j}.
\end{equation}
Then, we form the Lagrangian of this problem as  
\begin{equation}
    p^* = \min_{\|\weightmat_{1j}\|_F \leq 1} \min_{ \weightmat_{2j}, \vec{R}_i} \max_{\dualmat_i} \sum_{i=1}^n \mathcal{L}(\vec{R}_i, \labelmat_i) + \beta \sum_{j=1}^m  \|\weightmat_{2j}\|_F  + \sum_{i=1}^n \mathrm{trace}( \dualmat_i^\top \vec{R}_i) - \sum_{i=1}^n \mathrm{trace}\left(\dualmat_i^\top \sum_{j=1}^m g(\data_i; \weightmat_{1j}) \weightmat_{2j}\right).
\end{equation}
By Sion's minimax theorem, we can reverse the order of the outer maximum and minimum, and minimize this problem over \(\weightmat_{2j}\) and \(\vec{R}_i\). Defining the Fenchel conjugate of \(\mathcal{L}\) as \(\mathcal{L}^*(\dualmat_i, \labelmat_i) := \max_{\vec{R}_i} -\mathcal{L}(\vec{R}_i, \labelmat_i) + \mathrm{trace}( \dualmat_i^\top \vec{R}_i)\), we have
\begin{equation}
    p^* = \min_{\|\weightmat_{1j}\|_F \leq 1}  \max_{\dualmat_i} \min_{ \weightmat_{2j}} -\sum_{i=1}^n \mathcal{L}^*(\dualmat_i, \labelmat_i) + \beta \sum_{j=1}^m  \|\weightmat_{2j}\|_F - \sum_{i=1}^n \mathrm{trace}\left(\dualmat_i^\top \sum_{j=1}^m g(\data_i; \weightmat_{1j}) \weightmat_{2j}\right).
\end{equation}
Now, we solve over \(\weightmat_{2j}\) to obtain 
\begin{equation}
    p^* = \min_{\|\weightmat_{1j}\|_F \leq 1}  \max_{\dualmat_i: \|\sum_{i=1}^n \dualmat_i^\top g(\data_i; \weightmat_{1j}) \|_F \leq \beta }  -\sum_{i=1}^n \mathcal{L}^*(\dualmat_i, \labelmat_i) 
\end{equation}
Now, since Slater's condition holds because \(\beta >0\), as long as \(m \geq m^*\) where \(m^*\) is the dimension of the constraints (see the individual cases for examples) we are permitted to switch the order of the maximum and minimum to obtain the desired result \cite{shapiro2009semi}:
\begin{align}
    p^* &= \max_{\dualmat_i} -\sum_{i=1}^n \mathcal{L}^*(\dualmat_i, \labelmat_i) \nonumber \\
    \text{s.t.}~&\max_{\|\weightmat_1\|_F \leq 1} \|\sum_{i=1}^n \dualmat_i^\top g(\data_i; \weightmat_{1})\|_F \leq \beta.
\end{align}
\end{proof}
These two lemmas will prove invaluable in the subsequent proofs. 

\subsection{Proof of Theorem \ref{theo:lin_attention}}
We first note that for self-attention we remove the \(1/\sqrt{d}\) factor for simplicity. We apply Lemmas \ref{lem:rescaling} and \ref{lem:duality} to \eqref{eq:mhsa_opt_primal} with the linear activation function to obtain
\begin{align}
    p_{SA}^* &= \max_{\dualmat_i} -\sum_{i=1}^n \mathcal{L}^*\left(\dualmat_i, \labelmat_i\right) \nonumber \\
  \mathrm{s.t.}~& \max_{\|\weightmat_{1}\|_F \leq 1} \left\|\sum_{i=1}^n \dualmat_i^\top (\data_i \weightmat_{1} \data_i^\top)\data_i\right\|_F \leq \beta.
\end{align}
Due to the identity \(\vecop(\vec{A}\vec{B}\vec{C}) = (\vec{C}^\top \otimes \vec{A})\vecop(\vec{B})\) \cite{magnus2019matrix}, we can write this as
\begin{align}
    p_{SA}^* &= \max_{\dualmat_i} -\sum_{i=1}^n \mathcal{L}^*\left(\dualmat_i, \labelmat_i\right) \nonumber \\
  \mathrm{s.t.}~& \max_{\|\weight_{1}\|_2 \leq 1} \left\|\sum_{i=1}^n ( (\data_i^\top \otimes \dualmat_i^\top)(\data_i \otimes \data_i)) \weight_{1} \right\|_2 \leq \beta,
\end{align}
where \(\weight_1 = \vecop(\weightmat_1)\). We note that the norm constraint here has dimension \(dc\) and \(\weight_1\) has dimension \(d^2\), so by \cite{shapiro2009semi} this strong duality result from Lemma \ref{lem:duality} requires that \(m^* \leq \min\{d^2, dc\}\). We can further write this as 
\begin{align}
    p_{SA}^* &= \max_{\dualmat_i} -\sum_{i=1}^n \mathcal{L}^*\left(\dualmat_i, \labelmat_i\right) \nonumber \\
  \mathrm{s.t.}~& \max_{\|\vec{Z}\|_* \leq 1} \mathrm{trace}\left(\sum_{i=1}^n  (\data_i^\top \otimes \dualmat_i^\top)(\data_i \otimes \data_i) \vec{Z} \right) \leq \beta . 
\end{align}
Now, we form the Lagrangian, given by 
\begin{equation}
    p_{SA}^* =\max_{\dualmat_i} \min_{\|\vec{Z}\|_* \leq 1}  \min_{\lambda \geq 0}  -\sum_{i=1}^n \mathcal{L}^*(\dualmat_i, \labelmat_i) +  \lambda \left(\beta -  \sum_{i=1}^n \vecop((\data_i \otimes \data_i)\vec{Z})^\top \vecop\left(\data_i \otimes \dualmat_i \right)  \right).
\end{equation}
By Sion's minimax theorem, we are permitted to change the order of the maxima and minima, to obtain  
\begin{equation}
    p_{SA}^* = \min_{\lambda \geq 0} \min_{\|\vec{Z}\|_* \leq 1}\max_{\dualmat_i}  -\sum_{i=1}^n \mathcal{L}^*(\dualmat_i, \labelmat_i) +  \lambda \left(\beta -  \sum_{i=1}^n \vecop((\data_i \otimes \data_i)\vec{Z})^\top \vecop\left(\data_i \otimes \dualmat_i \right)  \right). 
\end{equation}
Now, defining \(\vec{K}_{c, s}\) as the \((c, s)\) commutation matrix we have the following identity \cite{magnus2019matrix}
\begin{align*}
    \vecop(\data_i \otimes  \dualmat_i)= \left((\vec{I}_{d} \otimes \vec{K}_{c, s})(\vecop(\data_i) \otimes \vec{I}_c) \otimes \vec{I}_{s}\right)\vecop(\dualmat_i). 
\end{align*}
Using this identity and maximizing over \(\dualmat_i\), we obtain 
\begin{equation}
    p_{SA}^* = \min_{\|\vec{Z}\|_* \leq 1} \min_{\lambda \geq 0}  \sum_{i=1}^n \mathcal{L}\left(\left((\vecop(\data_i)^\top \otimes \vec{I}_c)(\vec{I}_d \otimes \vec{K}_{s, c}) \otimes \vec{I}_{s} \right)\vecop((\data_i \otimes \data_i)(\lambda\vec{Z})), \vecop(\labelmat_i)\right) + \beta \lambda. 
\end{equation}
Rescaling such that \(\tilde{\vec{Z}} = \lambda \vec{Z}\), we obtain
\begin{equation}
    p_{SA}^* = \min_{\vec{Z} \in \mathbb{R}^{d^2 \times dc} }  \sum_{i=1}^n \mathcal{L}\left(\left((\vecop(\data_i)^\top \otimes \vec{I}_c)(\vec{I}_d \otimes \vec{K}_{s, c}) \otimes \vec{I}_{s} \right)\vecop((\data_i \otimes \data_i)\vec{Z}), \vecop(\labelmat_i)\right) + \beta  \|\vec{Z}\|_*. 
\end{equation}
It appears as though this is a very complicated function, but it actually simplifies greatly. In particular, one can write this as 
\begin{align}
    p_{SA}^* &= \min_{\vec{Z} \in \mathbb{R}^{d^2 \times dc} }  \sum_{i=1}^n \mathcal{L}\left(\hat{\labelmat}_i, \labelmat_i\right) + \beta  \|\vec{Z}\|_* \nonumber \\
    \hat{\labelmat}_i[o, p] &:= \sum_{k=1}^d \sum_{l=1}^d \sum_{t=1}^s \data_i[t, l]\data_i[t, k] \data_i[o, :]^\top \vec{Z}^{(k, l)}. 
\end{align}
Making any final simplifications, one obtains the desired result. 
\begin{align}
    p_{SA}^* = \min_{\vec{Z} \in \mathbb{R}^{d^2 \times dc} }  &\sum_{i=1}^n \mathcal{L}\left(\sum_{k=1}^d \sum_{\ell=1}^d \vec{G}_i[k, \ell] \data_i \vec{Z}^{(k, \ell)}, \labelmat_i\right) +\beta  \|\vec{Z}\|_*.
\end{align}
Lastly, we also demonstrate that there is a one-to-one mapping between the solution to \eqref{eq:mhsa_linear_convex} and \eqref{eq:mhsa_opt_primal}. In particular, imagine we have a solution \(\vec{Z}^*\) to \eqref{eq:mhsa_linear_convex} with optimal value \(p_{CVX}^*\). Let \(r := \mathrm{rank}(\vec{Z}^*)\), and take the SVD of \(\vec{Z}^*\) as \(\sum_{j=1}^r \sigma_j \vec{u}_j \vec{v}_j^\top\), where \(\vec{u}_j \in \mathbb{R}^{d^2}\) and \(\vec{v}_j \in \mathbb{R}^{dc}\). Let \(\mathrm{vec}^{-1}(\vec{u_j}) \in \mathbb{R}^{d \times d}\) be the result of taking chunks of \(d\)-length vectors from \(\vec{u}_j\) and stacking them in columns. Similarly, let \(\mathrm{vec}^{-1}(\vec{v_j}) \in \mathbb{R}^{c \times d}\) be the result of taking chunks of \(c\)-length vectors from \(\vec{v}_j\) and stacking them in columns. Furthermore we will let \(\mathrm{vec}^{-1}(\vec{u_j})_k\) be the \(k\)th column of \(\mathrm{vec}^{-1}(\vec{u_j})\). Then, recognize that 
\[{\vec{Z}^*}^{(k, \ell)} = \sum_{j=1}^r \sigma_j \mathrm{vec}^{-1}(\vec{u_j})_k \mathrm{vec}^{-1}(\vec{v_j})_{\ell}^\top \]
Thus, given \(\vec{Z}^*\), we can form a candidate solution to \eqref{eq:mhsa_opt_primal} as follows:
\begin{align*}
    \vec{Z}^* &= \sum_{j=1}^r \sigma_j \vec{u}_j \vec{v}_j^\top \\
    \hat{\weightmat}_{1j} &= \sqrt{\sigma_j}\mathrm{vec}^{-1}(\vec{u_j}) \\
    \hat{\weightmat}_{2j} &= \sqrt{\sigma_j}\mathrm{vec}^{-1}(\vec{v_j})^\top
\end{align*}
We then have 
\begin{align*}
    \hat{p}_{NCVX} &= \sum_{i=1}^n \mathcal{L}\left(\sum_{j=1}^r \data_i\hat{\weightmat}_{1j} \data_i^\top \data_i \hat{\weightmat}_{2j}, \vec{Y}_i \right) + \frac{\beta}{2}\sum_{j=1}^r \|\hat{\weightmat}_{1j}\|_F^2 + \|\hat{\weightmat}_{2j}\|_F^2\\
    &= \sum_{i=1}^n \mathcal{L}\left(\data_i \sum_{j=1}^r \sigma_j \mathrm{vec}^{-1}(\vec{u_j}) \vec{G}_i \mathrm{vec}^{-1}(\vec{v_j})^\top, \vec{Y}_i \right) + \frac{\beta}{2}\sum_{j=1}^r \|\sqrt{\sigma_j} \vec{u}_j\|_2^2 + \|\sqrt{\sigma_j} \vec{v}_j\|_2^2 \\
    &=  \sum_{i=1}^n \mathcal{L}\left( \sum_{k=1}^d \sum_{\ell=1}^d \data_i \sum_{j=1}^r \sigma_j \mathrm{vec}^{-1}(\vec{u_j})_k \vec{G}_i[k, \ell] \mathrm{vec}^{-1}(\vec{v_j})_{\ell}^\top, \vec{Y}_i \right) + \frac{\beta}{2}\sum_{j=1}^r \sigma_j + \sigma_j \\
    &=  \sum_{i=1}^n \mathcal{L}\left(\sum_{k=1}^d \sum_{\ell=1}^d  \vec{G}_i[k, \ell]\data_i  {\vec{Z}^*}^{(k, \ell)} , \vec{Y}_i \right) + \beta\|{\vec{Z}^*}\|_*\\
    &= p_{CVX}^*
\end{align*}
Thus, the two solutions match. Similarly, if we have the solution \((\vec{W}_{1j}^*, \vec{W}_{2j}^*)\) to \eqref{eq:mhsa_opt_primal}, we can form the equivalent optimal convex weights to \eqref{eq:mhsa_linear_convex} as 
\begin{align*}
    \vec{Z}^* = \sum_{j=1}^m \mathrm{vec}(\vec{W}_{1j}^*)\mathrm{vec}(\vec{W}_{2j}^*)^\top
\end{align*}
and the same proof can be demonstrated in reverse.  

\hfill\qedsymbol

\subsection{Proof of Corollary \ref{cor:lin_attention}}
We suppose that \(\vec{G}\) is block diagonal with \(B\) such blocks of size \(d_b\) such that \(\sum_b d_b = d\). Then, we have
\begin{equation}
    p_{SA}^* = \min_{\vec{Z} \in \mathbb{R}^{d^2 \times dc} }  \sum_{i=1}^n \mathcal{L}\left(\sum_{b=1}^B \sum_{j=1}^{d_b} \sum_{k=1}^{d_b} \vec{G}_{i}^{(b)}[j, k] \data_i \vec{Z}^{(b, j,k)}, \labelmat_i\right) + \beta  \|\vec{Z}\|_*
\end{equation}
Then, we have, for blocks \(\vec{Z}^{(b)} \in \mathbb{R}^{d_bd \times d_bc}\), 
\begin{equation}
    \|\vec{Z}\|_* = \max_{\|\vec{A}\|_2 \leq 1} \langle \vec{Z}, \vec{A} \rangle \geq \max_{\|\vec{A}^{(b)}\|_2 \leq 1} \sum_{b=1}^b \langle \vec{Z}^{(b)}, \vec{A}^{(b)}\rangle = \sum_{b=1}^B \|\vec{Z}^{(b)}\|_*.
\end{equation}
This lower bound for \(\vec{Z}\) is achievable without changing the fitting term, by letting
\begin{equation}
    \vec{Z} = \begin{bmatrix} \vec{Z}^{(1)} & 0 & 0 \\ 0 & \vec{Z}^{(2)} & 0 \\ & \cdots & \\ 0 & 0 & \vec{Z}^{(B)}\end{bmatrix}
\end{equation}
Thus, in the block diagonal case, we have
\begin{align}
    p_{SA}^* = \min_{\vec{Z}}  &\sum_{i=1}^n \mathcal{L}\left(\sum_{b=1}^B \sum_{k=1}^{d_b} \sum_{\ell=1}^{d_b} \vec{G}_{i}^{(b)}[k, \ell] \data_i \vec{Z}^{(b, k, \ell)}, \labelmat_i\right) \nonumber \\ & + \beta  \sum_{b=1}^B \|\vec{Z}^{(b)}\|_*
\end{align}
as desired. \hfill\qedsymbol

\subsection{Proof of Theorem \ref{theo:relu_attention}}
We first note that for self-attention we remove the \(1/\sqrt{d}\) factor for simplicity. We apply Lemmas \ref{lem:rescaling} and \ref{lem:duality} to \eqref{eq:mhsa_opt_primal} with the ReLU activation function to obtain
\begin{align}
 p_{SA}^* &= \max_{\dualmat_i} -\sum_{i=1}^n \mathcal{L}^*\left(\dualmat_i, \labelmat_i\right) \nonumber \\
  \mathrm{s.t.}~& \max_{\substack{\|\weightmat_1\|_F\leq 1}}  \left\| \sum_{i=1}^N \dualmat_i^\top  (\vec{X}_i \weightmat_1 \vec{X}_i^\top )_+ \vec{X}_i \right\|_F. 
\end{align}
We again apply \(\vecop(\vec{A}\vec{B}\vec{C}) = (\vec{C}^\top \otimes \vec{A})\vecop(\vec{B})\) \cite{magnus2019matrix} to obtain 
\begin{align}
 p_{SA}^* &= \max_{\dualmat_i} -\sum_{i=1}^n \mathcal{L}^*\left(\dualmat_i, \labelmat_i\right) \nonumber \\
  \mathrm{s.t.}~& \max_{\substack{\|\weight_1\|_2\leq 1}}  \left\| \sum_{i=1}^N (\data_i \otimes \dualmat_i^\top)  ((\vec{X}_i\otimes \data_i) \weight_1)_+\right\|_2. 
\end{align}
Now, let \(\vec{D}_j^{(i)} \in \mathbb{R}^{s^2 \times s^2}\) be the \(i\)th block of \(\vec{D}_j\), and enumerate over all possible hyperplane arrangements \(j\). Then, we have 
\begin{align}
 p_{SA}^* &= \max_{\dualmat_i} -\sum_{i=1}^n \mathcal{L}^*\left(\dualmat_i, \labelmat_i\right) \nonumber \\
  \mathrm{s.t.}~& \max_{\substack{\|\weight_1\|_2\leq 1 \\ j \in [P] \\ \vec{K}_j \weight_1 \geq 0}}  \left\| \sum_{i=1}^N (\data_i^\top  \otimes \dualmat_i^\top)  \vec{D}_{j}^{(i)}(\data_i \otimes \data_i) \weight_1  \right\|_2.
\end{align}
We note that the norm constraint here has dimension \(dc\) and \(\weight_1\) has dimension \(d^2\), so by \cite{shapiro2009semi, pilanci2020neural} this strong duality result from Lemma \ref{lem:duality} requires that \(m^* \leq n\min\{d^2, dc\}\). Now, using the concept of dual norm, this is equal to 
\begin{align}
 p_{SA}^* &= \max_{\dualmat_i} -\sum_{i=1}^n \mathcal{L}^*\left(\dualmat_i, \labelmat_i\right) \nonumber \\
  \mathrm{s.t.}~&\max_{\substack{\|\vec{g}\|_2 \leq 1 \\ \|\weight_1\|_2\leq 1 \\ j \in [P] \\ \vec{K}_j \weight_1 \geq 0}} \vec{g}^\top \sum_{i=1}^N (\data_i^\top  \otimes \dualmat_i^\top)  \vec{D}_{j}^{(i)}(\data_i \otimes \data_i) \weight_1 
\end{align}
We can also define sets \(\mathcal{C}_j := \{\vec{Z} = \vec{u}\vec{g}^\top \in \mathbb{R}^{d^2 \times dc}: \vec{K}_j \vec{u} \geq 0~\forall i,\;\|\vec{Z}\|_* \leq 1\}\). Then, we have 
\begin{align}
 p_{SA}^* &= \max_{\dualmat_i} -\sum_{i=1}^n \mathcal{L}^*\left(\dualmat_i, \labelmat_i\right) \nonumber \\
  \mathrm{s.t.}~&\max_{\substack{j \in [P] \\ \vec{Z} \in \mathcal{C}_j}} \mathrm{trace}\left( \sum_{i=1}^n (\data_i^\top  \otimes \dualmat_i^\top)  \vec{D}_{j}^{(i)}(\data_i \otimes \data_i) \vec{Z} \right)
\end{align}
Now, we simply need to form the Lagrangian and solve. The Lagrangian is given by
\begin{equation}
     p_{SA}^* = \max_{\dualmat_i} \min_{\lambda \geq 0} \min_{\vec{Z}_j \in \mathcal{C}_j} -\sum_{i=1}^n \mathcal{L}^*(\dualmat_i, \labelmat_i) + \sum_{j=1}^P \lambda_j \left(\beta -  \sum_{i=1}^n \vecop\left(\vec{D}_{j}^{(i)} (\data_i \otimes \data_i)\vec{Z}_j\right)^\top\vecop\left(\data_i \otimes \dualmat_i \right)  \right)
\end{equation}
We now can switch the order of max and min via Sion's minimax theorem and maximize over \(\dualmat_i\). Defining \(\vec{K}_{c, s}\) as the \((c, s)\) commutation matrix \cite{magnus2019matrix}:
\begin{align*}
    \vecop(\data_i \otimes  \dualmat_i)= \left((\vec{I}_{d} \otimes \vec{K}_{c, s})(\vecop(\data_i) \otimes \vec{I}_c) \otimes \vec{I}_{s}\right)\vecop(\dualmat_i)
\end{align*}
Maximizing over \(\dualmat_i\), we have 
\begin{equation}
    p_{SA}^* = \min_{\lambda \geq 0} \min_{\vec{Z}_j \in \mathcal{C}_j}  \sum_{i=1}^n \mathcal{L}\left(\sum_{j=1}^P \left((\vecop(\data_i)^\top \otimes \vec{I}_c)(\vec{I}_d \otimes \vec{K}_{s, c}) \otimes \vec{I}_{s} \right)\vecop\left(\vec{D}_{j}^{(i)} (\data_i \otimes \data_i)\lambda_j\vec{Z}_j\right), \vecop(\labelmat_i)\right) + \beta \sum_{j=1}^m \lambda_j. 
\end{equation}
Again rescaling \(\tilde{\vec{Z}}_j = \lambda_j\vec{Z}_j \), we have
\begin{equation}
    p_{SA}^* = \min_{\vec{Z}_j \in \mathbb{R}^{d^2 \times dc} }  \sum_{i=1}^n \mathcal{L}\left(\sum_{j=1}^P \left((\vecop(\data_i)^\top \otimes \vec{I}_c)(\vec{I}_d \otimes \vec{K}_{s, c}) \otimes \vec{I}_{s} \right)\vecop\left(\vec{D}_{j}^{(i)} (\data_i \otimes \data_i)\vec{Z}_j\right), \vecop(\labelmat_i)\right) + \beta \sum_{j=1}^m \|\vec{Z}_j\|_{*, \mathrm{K}_j}.
\end{equation}
It appears as though this is a very complicated function, but it actually simplifies greatly. In particular, one can write this as 
\begin{align}
    p_{SA}^* &= \min_{\vec{Z}_j \in \mathbb{R}^{d^2 \times dc} }  \sum_{i=1}^n \mathcal{L}\left(\hat{\labelmat}_i, \labelmat_i\right) + \beta  \|\vec{Z}\|_{*, \mathrm{K}_j} \nonumber \\
    \hat{\labelmat}_i[o, p] &:= \sum_{j=1}^P \sum_{k=1}^d \sum_{l=1}^d \sum_{t=1}^s \data_i[t, l]\data_i[t, k] \vec{D}_{j}^{(t, m)} \data_i[o, :]^\top \vec{Z}_j^{(k, l)}. 
\end{align}
Making any final simplifications, one obtains the desired result. 
\begin{align}
    p_{SA}^* &= \min_{\vec{Z}_j \in \mathbb{R}^{d^2 \times dc} } \sum_{i=1}^n \mathcal{L}\left(\sum_{j=1}^P \sum_{k=1}^d \sum_{\ell=1}^d \vec{G}_{i,j}^{(k, \ell)}\vec{X}_i\vec{Z}_j^{(k, \ell)}, \labelmat_i\right) + \beta \sum_{j=1}^m \|\vec{Z}_j\|_{*, \mathrm{K}_{j}}. 
\end{align}
Lastly, we also demonstrate that there is a one-to-one mapping between the solution to \eqref{eq:mhsa_relu_convex} and \eqref{eq:mhsa_opt_primal}. In particular, imagine we have a solution \(\{\vec{Z}_j^*\}_{j=1}^{P}\) to \eqref{eq:mhsa_relu_convex} with optimal value \(p_{CVX}^*\), where \(m^* \leq n\) are non-zero. Take the cone-constrained SVD of \(\vec{Z}_j^*\) as \(\sum_{x=1}^{r_j} \sigma_{jx} \vec{u}_{jx} \vec{v}_{jx}^\top\), where \(\vec{u}_{jx} \in \mathbb{R}^{d^2}\) and \((2\vec{D}_j - \vec{I}_{ns^2})\data \vec{u}_{jx} \geq 0\), and \(\vec{v}_j \in \mathbb{R}^{dc}\). Let \(\mathrm{vec}^{-1}(\vec{u_{jx}}) \in \mathbb{R}^{d \times d}\) be the result of taking chunks of \(d\)-length vectors from \(\vec{u}_{jx}\) and stacking them in columns. Similarly, let \(\mathrm{vec}^{-1}(\vec{v_{jx}}) \in \mathbb{R}^{c \times d}\) be the result of taking chunks of \(c\)-length vectors from \(\vec{v}_j\) and stacking them in columns. Furthermore we will let \(\mathrm{vec}^{-1}(\vec{u_j})_k\) be the \(k\)th column of \(\mathrm{vec}^{-1}(\vec{u_j})\). Then, recognize that 
\[{\vec{Z}^*}_j^{(k, \ell)} = \sum_{x=1}^r \sigma_{jx} \mathrm{vec}^{-1}(\vec{u_{jx}})_k \mathrm{vec}^{-1}(\vec{v_{jx}})_{\ell}^\top \]
Thus, given \(\vec{Z}^*\), we can form a candidate solution to \eqref{eq:mhsa_opt_primal} as follows:
\begin{align*}
    \vec{Z}_j^* &= \sum_{x=1}^{r_j} \sigma_{jx} \vec{u}_{jx} \vec{v}_{jx}^\top,\,(2\vec{D}_j - \vec{I}_{ns^2})\data \vec{u}_{jx} \geq 0,\, \|\vec{u}_{jx}\|_2 =1,\,\|\vec{v}_{jx}\|_2 =1 \\
    \hat{\weightmat}_{1jx} &= \sqrt{\sigma_{jx}}\mathrm{vec}^{-1}(\vec{u_{jx}}) \\
    \hat{\weightmat}_{2jx} &= \sqrt{\sigma_{jx}}\mathrm{vec}^{-1}(\vec{v_{jx}})^\top
\end{align*}
We then have 
\begin{align*}
    \hat{p}_{NCVX} &= \sum_{i=1}^n \mathcal{L}\left(\sum_{j=1}^{m^*} \sum_{x=1}^{r_j} (\data_i\hat{\weightmat}_{1jx} \data_i^\top)_+ \data_i \hat{\weightmat}_{2jx}, \vec{Y}_i \right) + \frac{\beta}{2}\sum_{j=1}^{m^*} \sum_{x=1}^{r_j} \|\hat{\weightmat}_{1jx}\|_F^2 + \|\hat{\weightmat}_{2jx}\|_F^2\\
    &= \sum_{i=1}^n \mathcal{L}\left( \sum_{j=1}^{m^*} \sum_{x=1}^{r_j} \sigma_{jx} \left( \mathrm{diag}^{-1}(\vec{D}_j^{(i)}) \odot (\data_i \mathrm{vec}^{-1}(\vec{u_{jx}}) \data_i^\top) \right) \data_i \mathrm{vec}^{-1}(\vec{v_{jx}})^\top, \vec{Y}_i \right) + \frac{\beta}{2}\sum_{j=1}^{m^*} \sum_{x=1}^{r_j}\|\sqrt{\sigma_{jx}} \vec{u}_{jx}\|_2^2 + \|\sqrt{\sigma_{jx}} \vec{v}_{jx}\|_2^2 \\
    &=  \sum_{i=1}^n \mathcal{L}\left( \sum_{k=1}^d \sum_{\ell=1}^d  \sum_{j=1}^{m^*} \sum_{x=1}^{r_j} \sigma_{jx} \vec{G}_{i, j}^{(k, \ell)} \data_i \mathrm{vec}^{-1}(\vec{u_j})_k  \mathrm{vec}^{-1}(\vec{v_j})_{\ell}^\top, \vec{Y}_i \right) + \frac{\beta}{2} \sum_{j=1}^{m^*} \sum_{x=1}^{r_j} \sigma_{jx}+ \sigma_{jx} \\
    &=  \sum_{i=1}^n \mathcal{L}\left(\sum_{k=1}^d \sum_{\ell=1}^d  \vec{G}_{i, j}^{(k, \ell)} \data_i  {\vec{Z}_j^*}^{(k, \ell)} , \vec{Y}_i \right) + \beta\sum_{j=1}^{m^*}\|\vec{Z}_j^*\|_{*, \mathrm{K}_j}\\
    &= p_{CVX}^*
\end{align*}
Thus, the two solutions match. 
\hfill\qedsymbol

\subsection{Proof of Theorem \ref{theo:linear_mlpmixer}}
We apply Lemmas \ref{lem:rescaling} and \ref{lem:duality} to \eqref{eq:mlp_mixer_primal} with the linear activation function to obtain
\begin{align}
    p_{MM}^* &= \max_{\dualmat_i} -\sum_{i=1}^n \mathcal{L}^* \left (\dualmat_i ,\labelmat_i \right) \nonumber \\ 
   \text{s.t.} &~ \max_{\substack{\|\weightmat_1\|_F\leq 1}} \|  \sum_{i=1}^n \dualmat_i^T \weightmat_1 \data_i  \|_F \leq \beta
\end{align}
We again apply \(\vecop(\vec{A}\vec{B}\vec{C}) = (\vec{C}^\top \otimes \vec{A})\vecop(\vec{B})\) \cite{magnus2019matrix} and maximize over \(\vecop(\weightmat_1)\) to obtain 
\begin{align*}
    p_{MM}^*  &= -\max_{\dualmat_i} \sum_{i=1}^n \mathcal{L}^* \left ( \dualmat_i, \labelmat_i  \right)\\
    &\mbox{s.t. } \| \sum_{i=1}^n \data_i^\top \otimes \dualmat_i^\top \|_2\leq \beta.
\end{align*}
We note that the norm constraint here has dimension \(dc\) and \(\vecop(\weightmat_1)\) has dimension \(s^2\), so by \cite{shapiro2009semi} this strong duality result from Lemma \ref{lem:duality} requires that \(m^* \leq \min\{s^2, dc\}\). The Lagrangian is given by
\begin{equation}
    p_{MM}^*  = \max_{\dualmat_i} \min_{\lambda \geq 0} \min_{\|\vec{Z}\|_*\leq 1} -\sum_{i=1}^n \mathcal{L}^*(\dualmat_i, \labelmat_i) + \sum_{j=1}^m \lambda_j \left(\beta -  \sum_{i=1}^n\mathrm{trace}\left( (\vec{I}_d \otimes \dualmat_i^T)(\data_i^\top \otimes \vec{I}_{s}) \vec{Z} \right)\right)
\end{equation}
which by Sion's minimax theorem and simplification can also be written as
\begin{equation}
   p_{MM}^*  = \min_{\lambda \geq 0} \min_{\|\vec{Z}\|_*\leq 1}  \max_{\dualmat_i} -\sum_{i=1}^n \mathcal{L}^*(\dualmat_i, \labelmat_i) + \sum_{j=1}^m \lambda_j \left(\beta -  \sum_{i=1}^n \vecop((\data_i^\top \otimes \vec{I}_{s})\vec{Z})^\top \vecop(\vec{I}_d \otimes \dualmat_i)\right)
\end{equation}
We define \(\vec{K}_{c, d}\) as the \((c, d)\) commutation matrix \cite{magnus2019matrix}:
\begin{align*}
    \vecop(\vec{I}_d \otimes  \dualmat_i)= \left((\vec{I}_{d} \otimes \vec{K}_{c, d})(\vecop(\vec{I}_d) \otimes \vec{I}_c) \otimes \vec{I}_{s}\right)\vecop(\dualmat_i)
\end{align*}
Maximizing over \(\dualmat_i\), followed by re-scaling \(\tilde{\vec{Z}} = \lambda \vec{Z}\) gives us
\begin{equation}
    p_{MM}^*  = \min_{\vec{Z} \in \mathbb{R}^{s^2 \times dc} }  \sum_{i=1}^n \mathcal{L}\left( \left((\vecop(\vec{I}_d)^\top \otimes \vec{I}_c)(\vec{I}_d \otimes \vec{K}_{d, c}) \otimes \vec{I}_{s} \right)\vecop\left((\data_i^\top \otimes \vec{I}_{s})\vec{Z} \right), \vecop(\labelmat_i)\right) + \beta \|\vec{Z}\|_{*}
\end{equation}
Making any final simplifications, one obtains the desired result.
\begin{align}
    p_{MM}^* &= \min_{\vec{Z} \in \mathbb{R}^{s^2\times dc} }  \sum_{i=1}^n \mathcal{L}\left(\begin{bmatrix} f_1(\vec{X}_i) & \cdots &  f_c(\vec{X}_i)  \end{bmatrix}, \labelmat_i\right)+ \beta \|\vec{Z}\|_{*} \nonumber \\
   f_p(\vec{X}_i)  &:= \vec{Z}^{(p)}\vecop(\data_i). 
\end{align}

Lastly, we also demonstrate that there is a one-to-one mapping between the solution to \eqref{eq:mlp_mixer_linear_convex} and \eqref{eq:mlp_mixer_primal}. In particular, we have a solution \(\vec{Z}^*\) to \eqref{eq:mlp_mixer_linear_convex} with optimal value \(p_{CVX}^*\). First, rearrange the solution to be of the form \(\tilde{\vec{Z}}^*\) of \eqref{eq:mlp_mixer_linear_convex_permuted}. Let \(r := \mathrm{rank}(\tilde{\vec{Z}}^*)\), and take the SVD of \(\tilde{\vec{Z}}^*\) as \(\sum_{j=1}^r \sigma_j \vec{u}_j \vec{v}_j^\top\), where \(\vec{u}_j \in \mathbb{R}^{s^2}\) and \(\vec{v}_j \in \mathbb{R}^{dc}\). Let \(\mathrm{vec}^{-1}(\vec{u_j}) \in \mathbb{R}^{s \times s}\) be the result of taking chunks of \(s\)-length vectors from \(\vec{u}_j\) and stacking them in columns. Similarly, let \(\mathrm{vec}^{-1}(\vec{v_j}) \in \mathbb{R}^{c \times d}\) be the result of taking chunks of \(c\)-length vectors from \(\vec{v}_j\) and stacking them in columns. Furthermore we will let \(\mathrm{vec}^{-1}(\vec{u_j})_t\) be the \(t\)th column of \(\mathrm{vec}^{-1}(\vec{u_j})\). Then, recognize that 
\[\tilde{\vec{Z}}^{*^{(t, k)}} = \sum_{j=1}^r \sigma_j \mathrm{vec}^{-1}(\vec{u_j})_t \mathrm{vec}^{-1}(\vec{v_j})_{k}^\top \]
Thus, given \(\vec{Z}^*\), we can form a candidate solution to \eqref{eq:mhsa_opt_primal} as follows:
\begin{align*}
    \tilde{\vec{Z}}^* &= \sum_{j=1}^r \sigma_j \vec{u}_j \vec{v}_j^\top \\
    \hat{\weightmat}_{1j} &= \sqrt{\sigma_j}\mathrm{vec}^{-1}(\vec{u_j}) \\
    \hat{\weightmat}_{2j} &= \sqrt{\sigma_j}\mathrm{vec}^{-1}(\vec{v_j})^\top
\end{align*}
We then have 
\begin{align*}
    \hat{p}_{NCVX} &= \sum_{i=1}^n \mathcal{L}\left(\sum_{j=1}^r \hat{\weightmat}_{1j} \data_i \hat{\weightmat}_{2j}, \vec{Y}_i \right) + \frac{\beta}{2}\sum_{j=1}^r \|\hat{\weightmat}_{1j}\|_F^2 + \|\hat{\weightmat}_{2j}\|_F^2\\
    &= \sum_{i=1}^n \mathcal{L}\left(\data_i \sum_{j=1}^r \sigma_j \mathrm{vec}^{-1}(\vec{u_j}) \data_i \mathrm{vec}^{-1}(\vec{v_j})^\top, \vec{Y}_i \right) + \frac{\beta}{2}\sum_{j=1}^r \|\sqrt{\sigma_j} \vec{u}_j\|_2^2 + \|\sqrt{\sigma_j} \vec{v}_j\|_2^2 \\
    &=  \sum_{i=1}^n \mathcal{L}\left( \sum_{t=1}^s \sum_{k=1}^d \sum_{j=1}^r \sigma_j \mathrm{vec}^{-1}(\vec{u_j})_t \data_i[t, k] \mathrm{vec}^{-1}(\vec{v_j})_{k}^\top, \vec{Y}_i \right) + \frac{\beta}{2}\sum_{j=1}^r \sigma_j + \sigma_j \\
    &=  \sum_{i=1}^n \mathcal{L}\left(\sum_{t=1}^s \sum_{k=1}^d  \data_i[k, \ell] \tilde{\vec{Z}}^{*^{(t, k)}}, \vec{Y}_i \right) + \beta\|\tilde{\vec{Z}}^*\|_*\\
    &= p_{CVX}^*
\end{align*}
Thus, the two solutions match.
\hfill\qedsymbol

\subsection{Proof of Theorem \ref{theo:relu_mlpmixer}}
We apply Lemmas \ref{lem:rescaling} and \ref{lem:duality} to \eqref{eq:mlp_mixer_primal} with the ReLU activation function to obtain
\begin{align}
    p_{MM}^* &= \max_{\vec{V}_i} -\sum_{i=1}^n\mathcal{L}^* \left (\dualmat_i,\mathbf{Y}_i \right) \nonumber \\ 
   \text{s.t.} &~ \max_{\substack{\|\weightmat_1\|_F\leq 1}}   \left\|\sum_{i=1}^n \dualmat_i^T (\weightmat_1 \data_i)_+ \right\|_F \leq \beta
\end{align}
This is equivalent to \cite{magnus2019matrix}
\begin{align}
    p_{MM}^* &= \max_{\vec{V}_i} -\sum_{i=1}^n\mathcal{L}^* \left (\dualmat_i,\mathbf{Y}_i \right) \nonumber \\ 
   \text{s.t.} &~ \max_{\substack{\|\weightmat_1\|_F\leq 1}}   \left\|\sum_{i=1}^n  (\vec{I}_d \otimes \dualmat_i^T)((\data_i^\top \otimes \vec{I}_s)\vecop(\weightmat_1)_+ \right\|_F \leq \beta
\end{align}
We note that the norm constraint here has dimension \(dc\) and \(\vecop(\weightmat_1)\) has dimension \(s^2\), so by \cite{shapiro2009semi, pilanci2020neural} this strong duality result from Lemma \ref{lem:duality} requires that \(m^* \leq n\min\{s^2, dc\}\). Now, let \(\vec{D}_j^{(i)} \in \mathbb{R}^{sd \times sd}\) be the \(i\)th block of \(\vec{D}_j\), and enumerate over all possible hyperplane arrangements \(j\). Then, we have 
\begin{align}
    p_{MM}^* &= \max_{\vec{V}_i} -\sum_{i=1}^n \mathcal{L}^* \left (\dualmat_i,\mathbf{Y}_i \right) \nonumber \\ 
   \text{s.t.} &~ \max_{\substack{\|\weight_1\|_2\leq 1 \\ j \in [P] \\ \vec{K}_j\weight_1 \geq 0}}   \left\|\sum_{i=1}^n (\vec{I}_d \otimes \dualmat_i^T) \vec{D}_j^{(i)} (\data_i^\top \otimes \vec{I}_{s}) \weight_1\right\|_2\leq \beta.
\end{align}

Now, using the concept of dual norm, this is equal to 
\begin{align}
 p_{MM}^* &= \max_{\vec{V}_i} -\sum_{i=1}^n \mathcal{L}^* \left (\dualmat_i,\mathbf{Y}_i \right) \nonumber \\ 
   \text{s.t.} &~
    \max_{\substack{\|\vec{g}\|_2 \leq 1 \\ \|\weight_1\|_2\leq 1 \\ j \in [P] \\ \vec{K}_j\weight_1 \geq 0}} \vec{g}^\top\sum_{i=1}^n (\vec{I}_d \otimes \dualmat_i^T) \vec{D}_j^{(i)} (\data_i^\top \otimes \vec{I}_{s}) \weight_1 \leq \beta
\end{align}
We can also define sets \(\mathcal{C}_j := \{\vec{Z} = \vec{u}\vec{g}^\top \in \mathbb{R}^{s^2 \times dc}: \vec{K}_j\vec{u} \geq 0,\;\|\vec{Z}\|_* \leq 1\}\). Then, we have 
\begin{align}
 p_{MM}^* &= \max_{\vec{V}_i} -\sum_{i=1}^n \mathcal{L}^* \left (\dualmat_i,\mathbf{Y}_i \right) \nonumber \\ 
   \text{s.t.} &~
    \max_{\substack{j \in [P] \\ \vec{Z} \in \mathcal{C}_j}} \mathrm{trace}\left( \sum_{i=1}^n (\vec{I}_d \otimes \dualmat_i^T) \vec{D}_j^{(i)} (\data_i^\top \otimes \vec{I}_{s})\vec{Z} \right) \leq \beta
\end{align}
Now, we simply need to form the Lagrangian and solve. The Lagrangian is given by
\begin{equation}
    p_{MM}^*= \max_{\dualmat_i} \min_{\lambda \geq 0} \min_{\vec{Z}_j \in \mathcal{C}_j} -\sum_{i=1}^n \mathcal{L}^*(\dualmat_i, \labelmat_i) + \sum_{j=1}^P \lambda_j \left(\beta -  \sum_{i=1}^n\mathrm{trace}\left( (\vec{I}_d \otimes \dualmat_i^T) \vec{D}_j^{(i)} (\data_i^\top \otimes \vec{I}_{s}) \vec{Z} \right)\right)
\end{equation}
We now can switch the order of max and min via Sion's minimax theorem and maximize over \(\dualmat_i\):
\begin{equation}
    p_{MM}^* = \min_{\lambda \geq 0} \min_{\vec{Z}_j \in \mathcal{C}_j}\max_{\dualmat_i}  -\sum_{i=1}^n \mathcal{L}^*(\dualmat_i, \labelmat_i) + \sum_{j=1}^P \lambda_j \left(\beta -  \sum_{i=1}^n \vecop\left(\vec{D}_{j}^{(i)}(\data_i^\top \otimes \vec{I}_{s})\vec{Z}_j \right)^\top \vecop\left(\vec{I}_d \otimes \dualmat_i \right)  \right)
\end{equation}
Now, defining \(\vec{K}_{c, d}\) as the \((c, d)\) commutation matrix:
\begin{align*}
    \vecop(\vec{I}_d \otimes  \dualmat_i)= \left((\vec{I}_{d} \otimes \vec{K}_{c, d})(\vecop(\vec{I}_d) \otimes \vec{I}_c) \otimes \vec{I}_{s}\right)\vecop(\dualmat_i)
\end{align*}
Solving over \(\dualmat_i\) yields
\begin{equation}
    p_{MM}^* = \min_{\lambda \geq 0}\min_{\vec{Z}_j \in \mathcal{C}_j}  \sum_{i=1}^n \mathcal{L}\left(\sum_{j=1}^P \left((\vecop(\vec{I}_d)^\top \otimes \vec{I}_c)(\vec{I}_d \otimes \vec{K}_{d, c}) \otimes \vec{I}_{s} \right)\vecop\left(\vec{D}_{j}^{(i)}(\data_i^\top \otimes \vec{I}_{s})\lambda_j\vec{Z}_j \right), \vecop(\labelmat_i)\right) + \beta \sum_{j=1}^P \lambda_j
\end{equation}
Re-scaling \(\tilde{\vec{Z}}_j = \lambda_j \vec{Z}_j\) gives us
\begin{equation}
    p_{MM}^* = \min_{\vec{Z}_j}  \sum_{i=1}^n \mathcal{L}\left(\sum_{j=1}^P \left((\vecop(\vec{I}_d)^\top \otimes \vec{I}_c)(\vec{I}_d \otimes \vec{K}_{d, c}) \otimes \vec{I}_{s} \right)\vecop\left(\vec{D}_{j}^{(i)}(\data_i^\top \otimes \vec{I}_{s})\vec{Z}_j \right), \vecop(\labelmat_i)\right) + \beta \sum_{j=1}^P \|\vec{Z}_j\|_{*, \mathrm{K}_j}.
\end{equation}
One can actually greatly simplify this result, and can re-write this as
\begin{align}
    p_{MM}^* &= \min_{\vec{Z}_j \in \mathbb{R}^{s^2\times dc} }  \sum_{i=1}^n \mathcal{L}\left(\begin{bmatrix} f_1(\vec{X}_i) & \cdots &  f_c(\vec{X}_i) \end{bmatrix} , \labelmat_i\right) + \beta \sum_{j=1}^P \|\vec{Z}_j\|_{*,\mathrm{K}_j}\\
    f_p(\vec{X}_i) &:= \sum_{j=1}^P \begin{bmatrix}  \vec{D}_{j}^{(i, 1)}\vec{Z}_j^{(p, 1)} \cdots   \vec{D}_{j}^{(i, d)}\vec{Z}_j^{(p, d)}\end{bmatrix}  \vecop(\data_i).
\end{align}

Lastly, we also demonstrate that there is a one-to-one mapping between the solution to \eqref{eq:mlp_mixer_relu_convex} and \eqref{eq:mlp_mixer_primal}. In particular, we have a solution \(\{\vec{Z}_j^*\}_{j=1}^{m^*}\) to \eqref{eq:mlp_mixer_relu_convex} with optimal value \(p_{CVX}^*\), where \(m^* \leq n\) are non-zero. First, rearrange the solution to be of the form \(\tilde{\vec{Z}}_j^*\) of \eqref{eq:mlp_mixer_relu_convex_permuted}. Take the cone-constrained SVD of \(\tilde{\vec{Z}}_j^*\) as \(\sum_{x=1}^{r_j} \sigma_{jx} \vec{u}_{jx} \vec{v}_{jx}^\top\), where \(\vec{u}_{jx} \in \mathbb{R}^{s^2}\) and \(\vec{v}_{jx} \in \mathbb{R}^{dc}\). Let \(\mathrm{vec}^{-1}(\vec{u_{jx}}) \in \mathbb{R}^{s \times s}\) be the result of taking chunks of \(s\)-length vectors from \(\vec{u}_{jx}\) and stacking them in columns. Similarly, let \(\mathrm{vec}^{-1}(\vec{v_{jx}}) \in \mathbb{R}^{c \times d}\) be the result of taking chunks of \(c\)-length vectors from \(\vec{v}_{jx}\) and stacking them in columns. Furthermore we will let \(\mathrm{vec}^{-1}(\vec{u_{jx}})_t\) be the \(t\)th column of \(\mathrm{vec}^{-1}(\vec{u_{jx}})\). Then, recognize that 
\[\tilde{\vec{Z}}_j^{*^{(t, k)}} = \sum_{x=1}^{r_j} \sigma_{jx} \mathrm{vec}^{-1}(\vec{u_{jx}})_t \mathrm{vec}^{-1}(\vec{v_{jx}})_{k}^\top \]
Thus, given \(\vec{Z}_j^*\), we can form a candidate solution to \eqref{eq:mhsa_opt_primal} as follows:
\begin{align*}
    \tilde{\vec{Z}}_j^* &= \sum_{x=1}^{rx} \sigma_{jx} \vec{u}_{jx} \vec{v}_{jx}^\top,\,(2\vec{D}_j - \vec{I}_{nd})\data \vec{u}_{jx} \geq 0,\, \|\vec{u}_{jx}\|_2 =1,\,\|\vec{v}_{jx}\|_2 =1 \\
    \hat{\weightmat}_{1jx} &= \sqrt{\sigma_j}\mathrm{vec}^{-1}(\vec{u_{jx}}) \\
    \hat{\weightmat}_{2jx} &= \sqrt{\sigma_j}\mathrm{vec}^{-1}(\vec{v_{jx}})^\top
\end{align*}
We then have 
\begin{align*}
    \hat{p}_{NCVX} &= \sum_{i=1}^n \mathcal{L}\left(\sum_{j=1}^{m^*} \sum_{x=1}^{r_j} (\hat{\weightmat}_{1jx} \data_i)_+ \hat{\weightmat}_{2jx}, \vec{Y}_i \right) + \frac{\beta}{2}\sum_{j=1}^{m^*} \sum_{x=1}^{r_j}\|\hat{\weightmat}_{1j}\|_F^2 + \|\hat{\weightmat}_{2j}\|_F^2\\
    &= \sum_{i=1}^n \mathcal{L}\left(\data_i \sum_{j=1}^{m^*} \sum_{x=1}^{r_j}\sigma_j (\mathrm{vec}^{-1}(\vec{u_j}) \data_i)_+ \mathrm{vec}^{-1}(\vec{v_j})^\top, \vec{Y}_i \right) + \frac{\beta}{2}\sum_{j=1}^{m^*} \sum_{x=1}^{r_j} \|\sqrt{\sigma_j} \vec{u}_j\|_2^2 + \|\sqrt{\sigma_j} \vec{v}_j\|_2^2 \\
    &=  \sum_{i=1}^n \mathcal{L}\left( \sum_{t=1}^s \sum_{k=1}^d \sum_{j=1}^{m^*} \sum_{x=1}^{r_j} \sigma_j (\mathrm{vec}^{-1}(\vec{u_j})_t \data_i[t, k])_+ \mathrm{vec}^{-1}(\vec{v_j})_{k}^\top, \vec{Y}_i \right) + \frac{\beta}{2}\sum_{j=1}^{m^*} \sum_{x=1}^{r_j} \sigma_j + \sigma_j \\
    &=  \sum_{i=1}^n \mathcal{L}\left(\sum_{j=1}^{m^*} \sum_{t=1}^s \sum_{k=1}^d  \data_i[k, \ell] \vec{D}_j^{(i, k)} \tilde{\vec{Z}}_j^{*^{(t, k)}}, \vec{Y}_i \right) + \beta \sum_{j=1}^{m^*} \|\tilde{\vec{Z}}_j^*\|_*\\
    &= p_{CVX}^*
\end{align*}
Thus, the two solutions match.
\hfill\qedsymbol

\subsection{Proof of Lemma \ref{lem:fno}}
The expression \eqref{eq:fno_full_problem_def} can equivalently be written as
\begin{align}
    f_{FN}(\data_i) = \sigma\left(\left(\vec{F}^{-1} 
    \begin{bmatrix}{\vec{F}}_1^\top \vec{X}_i & &  \\  & \ddots \\ & & {\vec{F}}_{s}^\top \vec{X}_i \end{bmatrix} \begin{bmatrix} \vec{V}^{(1)} \\ \cdots \\ \vec{V}^{(s)} \end{bmatrix} \right)\weightmat_1 \right)\weightmat_2
\end{align}
and further as
\begin{align}
     f_{FN}(\data_i) = \sigma\left(\left({\vec{F}}^{-1} 
    \begin{bmatrix}{\vec{F}}_1^\top \vec{X}_i & &  \\  & \ddots \\ & & {\vec{F}}_{s}^\top \vec{X}_i \end{bmatrix} ({\vec{F}} \otimes \vec{I}_d) \begin{bmatrix} \vec{L}^{(1)} \\ \cdots \\ \vec{L}^{(s)} \end{bmatrix} \right)\weightmat_1 \right)\weightmat_2,
\end{align}
which simplifies to 
\begin{align}
    f_{FN}(\data_i) = \sigma\left( \left(\begin{bmatrix} \data_i  & {\data_i}_{(1)} \cdots  & {\data_i}_{(s)}  \end{bmatrix} \vec{L} \right)\weightmat_1 \right)\weightmat_2,
\end{align}
where \({\data_i}_{(u)}\) is \(\data_i\) circularly shifted by \(u\) spots in its first dimension, and \(\vec{K}\) has been reshaped to the form \(\mathbb{R}^{sd \times d}\). Noting that we can merge \(\vec{L}\weightmat_1\) into one matrix without losing any expressibility, we obtain
\begin{align}
     f_{FN}(\data_i) = \sigma\left(\mathrm{circ}(\vec{X}_i)\weightmat_1 \right)\weightmat_2 = \sum_{j=1}^m \sigma\left(\mathrm{circ}(\vec{X}_i)\weight_{1j} \right)\weight_{2j}^\top 
\end{align}
as desired. \hfill\qedsymbol
\subsection{Proof of Theorem \ref{theo:linear_fno}}
We start with the problem 
\begin{equation}
    p_{FN}^* = \min_{\weight_{1j}, \weight_{2j}} \sum_{i=1}^n \mathcal{L}\left( \sum_{j=1}^m \mathrm{circ}(\vec{X}_i)\weight_{1j} \weight_{2j}^\top , \labelmat_i\right) + \frac{\beta}{2} \sum_{j=1}^m \|\weight_{1j}\|_2^2 + \|\weight_{2j}\|_2^2
\end{equation}
We now apply Lemmas \ref{lem:rescaling} and \ref{lem:duality} to obtain 
\begin{align}
    p_{FN}^* &= \max_{\dualmat_i} -\sum_{i=1}^n \mathcal{L}^*(\dualmat_i, \labelmat_i) \nonumber \\
    \mathrm{s.t.}~&\max_{\|\weight_1\|_2 \leq 1} \|\sum_{i=1}^n \dualmat_i^\top \mathrm{circ}(\data_i) \weight_1\|_2 \leq \beta
\end{align}
We note that the norm constraint here has dimension \(c\) and \(\weight_1\) has dimension \(sd\), so by \cite{shapiro2009semi, pilanci2020neural} this strong duality result from Lemma \ref{lem:duality} requires that \(m^* \leq \min\{sd, c\}\). This is equivalent to 
\begin{align}
    p_{FN}^* &= \max_{\dualmat_i} -\sum_{i=1}^n \mathcal{L}^*(\dualmat_i, \labelmat_i) \nonumber \\
    \mathrm{s.t.}~& \|\sum_{i=1}^n \dualmat_i^\top \mathrm{circ}(\data_i)\|_2 \leq \beta
\end{align}
We form the Lagrangian as 
\begin{align}
    p_{FN}^* &= \max_{\dualmat_i}\min_{\lambda \geq 0}\min_{\|\vec{Z}\|_* \leq 1} -\sum_{i=1}^n \mathcal{L}^*(\dualmat_i, \labelmat_i) + \lambda(\beta - \mathrm{trace}(\vec{Z}^\top \sum_{i=1}^n \mathrm{circ}(\data_i)^\top \dualmat_i )).
\end{align}
We switch the order of the maximum and minimum using Sion's minimax theorem and maximize over \(\dualmat_i\)
\begin{align}
    p_{FN}^* &= \min_{\lambda \geq 0}\min_{\|\vec{Z}\|_* \leq 1} \sum_{i=1}^n \mathcal{L}(\mathrm{circ}(\data_i)\vec{Z}, \labelmat_i) + \beta \lambda.
\end{align}
Lastly, we rescale \(\tilde{\vec{Z}} = \lambda \vec{Z}\) to obtain 
\begin{align}
    p_{FN}^* &= \min_{\vec{Z} \in \mathbb{R}^{sd \times c} }  \sum_{i=1}^n \mathcal{L}\left(\mathrm{circ}(\vec{X}_i) \vec{Z}, \labelmat_i\right) + \beta \|\vec{Z}\|_{*}. 
\end{align}
as desired. \hfill\qedsymbol

\subsection{Proof of Theorem \ref{theo:relu_fno}}
We now apply Lemmas \ref{lem:rescaling} and \ref{lem:duality} with the ReLU activation function to obtain 
\begin{align}
    p_{FN}^* &= \max_{\dualmat_i} -\sum_{i=1}^n \mathcal{L}^*(\dualmat_i, \labelmat_i) \nonumber \\
    \mathrm{s.t.}~&\max_{\|\weight_1\|_2 \leq 1} \|\sum_{i=1}^n \dualmat_i^\top (\mathrm{circ}(\data_i) \weight_1)_+\|_2 \leq \beta
\end{align}
We note that the norm constraint here has dimension \(c\) and \(\weight_1\) has dimension \(sd\), so by \cite{shapiro2009semi, pilanci2020neural} this strong duality result from Lemma \ref{lem:duality} requires that \(m^* \leq n\min\{sd, c\}\). We introduce hyperplane arrangements \(\vec{D}_j\) and enumerate over all of them, yielding
\begin{align}
    p_{FN}^* &= \max_{\dualmat_i} -\sum_{i=1}^n \mathcal{L}^*(\dualmat_i, \labelmat_i) \nonumber \\
    \mathrm{s.t.}~&\max_{\substack{\|\weight_1\|_2 \leq 1 \\ j \in [P] \\ \vec{K}_j \weight_1 \geq 0}} \|\sum_{i=1}^n \dualmat_i^\top \vec{D}_j^{(i)} \mathrm{circ}(\data_i) \weight_1\|_2 \leq \beta.
\end{align}
Using the concept of dual norm, this is equivalent to 
\begin{align}
    p_{FN}^* &= \max_{\dualmat_i} -\sum_{i=1}^n \mathcal{L}^*(\dualmat_i, \labelmat_i) \nonumber \\
    \mathrm{s.t.}~&\max_{\substack{\|\vec{g}\|_2 \leq 1 \\ \|\weight_1\|_2 \leq 1 \\ j \in [P] \\ \vec{K}_j \weight_1 \geq 0}} \vec{g}^\top \sum_{i=1}^n \dualmat_i^\top \vec{D}_j^{(i)} \mathrm{circ}(\data_i) \weight_1 \leq \beta
\end{align}
We can also define sets \(\mathcal{C}_j := \{\vec{Z} = \vec{u}\vec{g}^\top \in \mathbb{R}^{s^2 \times dc}: \vec{K}_j\vec{u} \geq 0,\;\|\vec{Z}\|_* \leq 1\}\). Then, we have 
\begin{align}
    p_{FN}^* &= \max_{\dualmat_i} -\sum_{i=1}^n \mathcal{L}^*(\dualmat_i, \labelmat_i) \nonumber \\
    \mathrm{s.t.}~&\max_{\substack{ j \in [P] \\ \vec{Z} \in \mathcal{C}_j}} \mathrm{trace}\left(\sum_{i=1}^n \dualmat_i^\top \vec{D}_j^{(i)} \mathrm{circ}(\data_i) \vec{Z}\right) \leq \beta
\end{align}
We form the Lagrangian as 
\begin{align}
    p_{FN}^* &= \max_{\dualmat_i}\min_{\lambda \geq 0}\min_{\vec{Z}_j \in \mathcal{C}_j} -\sum_{i=1}^n \mathcal{L}^*(\dualmat_i, \labelmat_i) + \sum_{j=1}^P \lambda_j(\beta - \mathrm{trace}(\vec{Z}_j^\top \sum_{i=1}^n \vec{D}_j^{(i)} \mathrm{circ}(\data_i)^\top \dualmat_i )).
\end{align}
We switch the order of the maximum and minimum using Sion's minimax theorem and maximize over \(\dualmat_i\)
\begin{align}
    p_{FN}^* &= \min_{\lambda_j \geq 0}\min_{\vec{Z}_j \in \mathcal{C}_j} \sum_{i=1}^n \mathcal{L}(\sum_{j=1}^P \vec{D}_j^{(i)} \mathrm{circ}(\data_i)\vec{Z}_j, \labelmat_i) + \beta \sum_{j=1}^P \lambda_j.
\end{align}
Lastly, we rescale \(\tilde{\vec{Z}}_j = \lambda_j \vec{Z}_j\) to obtain 
\begin{align}
    p_{FN}^* &= \min_{\vec{Z}_j \in \mathbb{R}^{sd \times c} }  \sum_{i=1}^n \mathcal{L}\left(\sum_{j=1}^P \vec{D}_j^{(i)} \mathrm{circ}(\vec{X}_i) \vec{Z}_j, \labelmat_i\right) + \beta \sum_{j=1}^P \|\vec{Z}_j\|_{*, \mathrm{K}_j}. 
\end{align}
as desired. \hfill\qedsymbol

\subsection{Proof of Lemma \ref{lem:bfno}}
The expression \eqref{eq:bfno_fn} can equivalently be written as
\begin{align}
    f_{BFN}(\data_i) = \sigma\left({\vec{F}}^{-1} 
    \begin{bmatrix}{\vec{F}}_1^\top \vec{X}_i & &  \\  & \ddots \\ & & {\vec{F}}_{s}^\top \vec{X}_i \end{bmatrix} \begin{bmatrix} \vec{V}^{(1)} \\ \cdots \\ \vec{V}^{(s)} \end{bmatrix} \right)\weightmat_2
\end{align}
and further as
\begin{align}
     f_{BFN}(\data_i) = \sigma\left({\vec{F}}^{-1} 
    \begin{bmatrix}{\vec{F}}_1^\top \vec{X}_i & &  \\  & \ddots \\ & & {\vec{F}}_{s}^\top \vec{X}_i \end{bmatrix} ({\vec{F}} \otimes \vec{I}_d) \begin{bmatrix} \vec{L}^{(1)} \\ \cdots \\ \vec{L}^{(s)} \end{bmatrix}\right)\weightmat_2,
\end{align}
which simplifies to 
\begin{align}
    f_{BFN}(\data_i) = \sigma\left(\begin{bmatrix} \data_i  & {\data_i}_{(1)} \cdots  & {\data_i}_{(s)}  \end{bmatrix} \vec{L} \right)\weightmat_2,
\end{align}
where \({\data_i}_{(u)}\) is \(\data_i\) circularly shifted by \(u\) spots in its first dimension, and \(\vec{L}\) has been reshaped to the form \(\mathbb{R}^{sd \times d}\). Without loss of generality, noting the block diagonal structure over blocks of \(\vec{L}\), we permute the rows and corresponding columns of \(\vec{L}\) so that \(\vec{L}\) is a block diagonal matrix
\begin{align}
     f_{BFN}(\data_i) = \sigma\left(\begin{bmatrix} \mathrm{circ}(\data_i^{(1)}) \cdots  & \mathrm{circ}(\data_i^{(B)}) \end{bmatrix} \vec{L}\right)\weightmat_2.
\end{align}
Simplifying the term inside, we have
\begin{align}
     f_{BFN}(\data_i) = \sigma\left(\begin{bmatrix} \mathrm{circ}(\data_i^{(1)})\vec{L}^{(1)} \cdots  & \mathrm{circ}(\data_i^{(B)})\vec{L}^{(B)} \end{bmatrix}\right)\weightmat_2.
\end{align}
Now, assuming \(\sigma\) is applied elementwise, we have 
\begin{align}
     f_{BFN}(\data_i) =\begin{bmatrix} \sigma(\mathrm{circ}(\data_i^{(1)})\vec{L}^{(1)}) \cdots  & \sigma(\mathrm{circ}(\data_i^{(B)})\vec{L}^{(B)}) \end{bmatrix}\weightmat_2.
\end{align}
Lastly, we use the block structure on \(\vec{W}_2\) to obtain 
\begin{align}
     f_{BFN}(\data_i) =\begin{bmatrix} \sigma(\mathrm{circ}(\data_i^{(1)})\vec{L}^{(1)})\weightmat_2^{(1)} \cdots  & \sigma(\mathrm{circ}(\data_i^{(B)})\vec{L}^{(B)})\weightmat_2^{(B)}  \end{bmatrix}
\end{align}
which we can write equivalently as 
\begin{align}
    f_{BFN}(\vec{X}_i) &= \begin{bmatrix} f_{BFN}^{(1)}(\vec{X}_i) & \cdots & f_{BFN}^{(B)}(\vec{X}_i) \end{bmatrix}\\
    f_{BFN}^{(b)}(\vec{X}_i) &= \sum_{j=1}^m \sigma\left(\mathrm{circ}(\vec{X}_i^{(b)})  \weight_{1bj}\right)\weight_{2bj}^\top \nonumber
\end{align}
as desired. \hfill\qedsymbol

\subsection{Proof of Theorem \ref{theo:linear_bfno}}
We now apply Lemmas \ref{lem:rescaling} and \ref{lem:duality} with ReLU activation obtain 
\begin{align}
    p_{BFN}^* &= \max_{\dualmat_i} -\sum_{i=1}^n \mathcal{L}^*(\dualmat_i, \labelmat_i) \nonumber \\
    \mathrm{s.t.}~&\max_{\|\weight_{1b}\|_2 \leq 1} \|\sum_{i=1}^n {\dualmat_i^{(b)}}^\top \mathrm{circ}(\data_i^{(b)}) \weight_{1b}\|_2 \leq \beta~\forall b \in [B]
\end{align}
We note that the norm constraint here has dimension \(c/B\) and \(\weight_{1b}\) has dimension \(sd/B\), so by \cite{shapiro2009semi, pilanci2020neural} this strong duality result from Lemma \ref{lem:duality} requires that \(m^* \leq 1/B\min\{sd, c\}\). This is equivalent to 
\begin{align}
    p_{BFN}^* &= \max_{\dualmat_i} -\sum_{i=1}^n \mathcal{L}^*(\dualmat_i, \labelmat_i) \nonumber \\
    \mathrm{s.t.}~& \|\sum_{i=1}^n {\dualmat_i^{(b)}}^\top \mathrm{circ}(\data_i^{(b)}) \|_2 \leq \beta ~\forall b \in [B]
\end{align}
We form the Lagrangian as 
\begin{align}
    p_{BFN}^* &= \max_{\dualmat_i}\min_{\lambda \geq 0}\min_{\|\vec{Z}_b\|_* \leq 1} -\sum_{i=1}^n \mathcal{L}^*(\dualmat_i, \labelmat_i) + \sum_{b=1}^B \lambda_b(\beta - \mathrm{trace}(\vec{Z}_b^\top \sum_{i=1}^n \mathrm{circ}(\data_i^{(b)})^\top \dualmat_i^{(b)} )).
\end{align}
We switch the order of the maximum and minimum using Sion's minimax theorem and maximize over \(\dualmat_i^{(b)}\)
\begin{align}
    p_{BFN}^* &= \min_{\lambda \geq 0}\min_{\|\vec{Z}_b\|_* \leq 1} \sum_{i=1}^n \mathcal{L}(\begin{bmatrix} \lambda_1 \mathrm{circ}(\data_i^{(1)})\vec{Z}_1 & \cdots & \lambda_B \mathrm{circ}(\data_i^{(B)}) \vec{Z}_B \end{bmatrix}, \labelmat_i) + \beta \sum_{j=1}^B \lambda_b.
\end{align}
Lastly, we rescale \(\tilde{\vec{Z}}_b = \lambda_b \vec{Z}_b\) to obtain 
\begin{align}
    p_{BFN}^* &= \min_{\vec{Z} \in \mathbb{R}^{sd \times c} }  \sum_{i=1}^n \mathcal{L}\left(\begin{bmatrix} \mathrm{circ}(\data_i^{(1)}) \vec{Z}_1 & \cdots & \mathrm{circ}(\data_i^{(B)})  \vec{Z}_B \end{bmatrix}, \labelmat_i\right) + \beta \|\vec{Z}\|_{*}. 
\end{align}
as desired. \hfill\qedsymbol

\subsection{Proof of Theorem \ref{theo:relu_bfno}}

We now apply Lemmas \ref{lem:rescaling} and \ref{lem:duality} with the ReLU activation function to obtain 
\begin{align}
    p_{BFN}^* &= \max_{\dualmat_i} -\sum_{i=1}^n \mathcal{L}^*(\dualmat_i, \labelmat_i) \nonumber \\
    \mathrm{s.t.}~&\max_{\|\weight_{1b}\|_2 \leq 1} \|\sum_{i=1}^n {\dualmat_i^{(b)}}^\top  (\mathrm{circ}(\data_i^{(b)}) \weight_{1b})_+\|_2 \leq \beta~\forall b \in [B]
\end{align}
We note that the norm constraint here has dimension \(c/B\) and \(\weight_{1b}\) has dimension \(sd/B\), so by \cite{shapiro2009semi, pilanci2020neural} this strong duality result from Lemma \ref{lem:duality} requires that \(m^* \leq n/B\min\{sd, c\}\). We introduce hyperplane arrangements \(\vec{D}_{b, j}\) and enumerate over all of them, yielding
\begin{align}
    p_{BFN}^* &= \max_{\dualmat_i} -\sum_{i=1}^n \mathcal{L}^*(\dualmat_i, \labelmat_i) \nonumber \\
    \mathrm{s.t.}~&\max_{\substack{\|\weight_{1b}\|_2 \leq 1 \\ j \in [P_b] \\ \vec{K}_{b, j} \weight_{1b} \geq 0}} \|\sum_{i=1}^n {\dualmat_i^{(b)}}^\top \vec{D}_{b, j}^{(i)} \mathrm{circ}(\data_i^{(b)}) \weight_{1b}\|_2 \leq \beta~\forall b \in [B].
\end{align}
Using the concept of dual norm, this is equivalent to 
\begin{align}
    p_{BFN}^* &= \max_{\dualmat_i} -\sum_{i=1}^n \mathcal{L}^*(\dualmat_i, \labelmat_i) \nonumber \\
    \mathrm{s.t.}~&\max_{\substack{\|\vec{g}_b\|_2 \leq 1 \\ \|\weight_{1b}\|_2 \leq 1 \\ j \in [P_b] \\ \vec{K}_{b, j} \weight_{1b} \geq 0}} \vec{g}^\top \sum_{i=1}^n {\dualmat_i^{(b)}}^\top \vec{D}_{b, j}^{(i)} \mathrm{circ}(\data_i^{(b)}) \weight_{1b} \leq \beta~\forall b \in [B].
\end{align}
We can also define sets \(\mathcal{C}_{b,j} := \{\vec{Z} = \vec{u}\vec{g}^\top \in \mathbb{R}^{s^2 \times dc}: \vec{K}_{b,j}\vec{u} \geq 0,\;\|\vec{Z}\|_* \leq 1\}\). Then, we have 
\begin{align}
    p_{BFN}^* &= \max_{\dualmat_i} -\sum_{i=1}^n \mathcal{L}^*(\dualmat_i, \labelmat_i) \nonumber \\
    \mathrm{s.t.}~&\max_{\substack{ j \in [P_b] \\ \vec{Z} \in \mathcal{C}_{b,j}}} \mathrm{trace}\left(\sum_{i=1}^n {\dualmat_i^{(b)}}^\top \vec{D}_{b, j}^{(i)} \mathrm{circ}(\data_i^{(b)}) \vec{Z}\right) \leq \beta~\forall b \in [B].
\end{align}
We form the Lagrangian as 
\begin{align}
    p_{BFN}^* &= \max_{\dualmat_i}\min_{\lambda \geq 0}\min_{\vec{Z}_{b,j} \in \mathcal{C}_{b,j}} -\sum_{i=1}^n \mathcal{L}^*(\dualmat_i, \labelmat_i) + \sum_{b=1}^B \sum_{j=1}^{P_b} \lambda_{b,j}(\beta - \mathrm{trace}(\vec{Z}_{b,j}^\top \sum_{i=1}^n \vec{D}_{b,j}^{(i)} \mathrm{circ}(\data_i^{(b)})^\top \dualmat_i^{(b)} )).
\end{align}
We switch the order of the maximum and minimum using Sion's minimax theorem and maximize over \(\dualmat_i\)
\begin{align}
    p_{BFN}^* &= \min_{\lambda_j \geq 0}\min_{\vec{Z}_{b,j} \in \mathcal{C}_{b,j}} \sum_{i=1}^n \mathcal{L}(\begin{bmatrix} \sum_{j=1}^{P_1}\lambda_{1, j} \vec{D}_{1,j}\mathrm{circ}(\vec{X}_i^{(1)})\vec{Z}_{1,j} & \cdots &  \sum_{j=1}^{P_B} \lambda_{B, j}\vec{D}_{B,j}\mathrm{circ}(\vec{X}_i^{(B)})\vec{Z}_{B,j} \end{bmatrix}, \labelmat_i) + \beta \sum_{b=1}^B \sum_{j=1}^{P_b}\lambda_{b,j}.
\end{align}
Lastly, we rescale \(\tilde{\vec{Z}}_{b,j} = \lambda_{b,j} \vec{Z}_{b,j}\) to obtain
\begin{align}
    p_{BFN}^* &= \min_{\vec{Z}_{b, j} }  \sum_{i=1}^n \mathcal{L}\left(\begin{bmatrix} \sum_{j=1}^{P_1} \vec{D}_{1,j}\mathrm{circ}(\vec{X}_i^{(1)})\vec{Z}_{1,j} & \cdots &  \sum_{j=1}^{P_B} \vec{D}_{B,j}\mathrm{circ}(\vec{X}_i^{(B)})\vec{Z}_{B,j} \end{bmatrix}, \labelmat_i\right) + \beta \sum_{b=1}^B \sum_{j=1}^{P_b} \|\vec{Z}_{j, b}\|_{*, \mathrm{K}_{b,j}},
\end{align}
as desired. \hfill\qedsymbol

\section{Experimental Details}\label{sec:appendix}
\begin{table}
\caption{Parameter count for convex heads used for the experiments in Table \ref{tab:results} .}
\label{tab:params}
\begin{center}
\begin{small}
\begin{sc}
\begin{tabular}{lcccc}
\toprule
Convex Head & Act. & Params\\
\midrule
Self-Attention & \multirow{6}*{Linear}  & 50.5M \\
MLP-Mixer & &  98.9M \\
B-FNO & &  392K\\
FNO & & 1.96M \\
MLP & &  1.96M \\
Linear & & 10K\\
\midrule
Self-Attention & \multirow{5}*{ReLU} & 253M\\
MLP-Mixer & &  196M\\
B-FNO & & 39.2M\\
FNO & & 196M \\
MLP & & 196M \\ 
\bottomrule
\end{tabular}
\end{sc}
\end{small}
\end{center}
\vskip -0.1in
\end{table}
All heads were trained on two NVIDIA 1080 Ti GPUs using the Pytorch deep learning library \cite{paszke2019pytorch}. For our backbone, we used pre-trained weights from the Pytorch Image Models library \cite{rw2019timm}. For all experiments, we trained each head for 70 epochs, and used a regularization parameter of \(\beta = 2 \times 10^{-2}\), the AdamW optimizer \cite{loshchilov2017decoupled}, and a cosine learning rate schedule with a warmup of three epochs with warmup learning rate of \(2 \times 10^{-7}\), an initial learning rate chosen based on training accuracy of either \(5 \times 10^{-3}\) or \(10^{-4}\), and a final learning rate of \(2 \times 10^{-2}\) times the initial learning rate. Data augmentation was performed using AutoAugment \cite{cubuk2018autoaugment}, along with color jittering, label smoothing, and training data interpolation. All heads aside from the self-attention head were trained using a batch size of 100, whereas the self-attention head was trained with a batch size of 20. For all ReLU heads, we choose a number of neurons such that the number of parameters across FNO, MLP-Mixer, self-attention, and MLPs are roughly equal. We provide information about the number of parameters in each head in Table \ref{tab:params}. \\\\
As an ablation, we also studied the effect of changing the backbone architecture on the results of this experiment. In particular, as a backbone, we also tried using a ViT-base model \cite{dosovitskiy2020image} with \(16 \times 16\) patches pre-trained on ImageNet-1k images of size \(224 \times 224\) (\(s=196\), \(d=768\)). Then, we followed the same average pooling approach as for the gMLP backbone experiments, and kept all other network parameters the same. The results of this experiment are summarized in Table \ref{tab:results_vit}. \\\\
We see from this table that some of the general results from Table \ref{tab:results} still hold, though in this case MLP-Mixer architectures and self-attention architectures are roughly equivalent in performance. We suspect that one major reason for the larger gap between the two methods in Table \ref{tab:results} could be due to the backbone architecture, since the gMLP architecture is MLP-based, as opposed to ViT which is self-attention based. Thus, we speculate that adding an additional MLP-Mixer head to gMLP may be more concordant with the features extracted from the gMLP backbone, whereas the inverse is true for the ViT backbone. \\\\
In order to avoid the heavy computational cost of nuclear norm minimization for all considered convex models (besides ``linear", which is just a logistic regression with standard weight decay), we rely on the Burer-Monteiro factorization \cite{burer2005local}. In particular, if one has the problem 
\begin{equation}\label{eq:bm_orig}
    \min_{\vec{Z} \in \mathbb{R}^{a \times c}} \sum_{i=1}^n \mathcal{L}(f_i(\vec{Z}), \vec{Y}_i) + \beta \|\vec{Z}\|_*,  
\end{equation}
we know that this is equivalent to 
\begin{equation}\label{eq:bm_ncvx}
    \min_{\vec{U} \in \mathbb{R}^{a \times b}, \vec{V} \in \mathbb{R}^{c \times b}} \sum_{i=1}^n \mathcal{L}(f_i(\vec{U}\vec{V}^\top ), \vec{Y}_i) + \frac{\beta}{2}\left( \|\vec{U}\|_F^2 + \|\vec{V}\|_F^2\right),
\end{equation}
granted that the solution \(\vec{Z}^*\) to \eqref{eq:bm_orig} has a rank less than the new latent dimension \(b\) \cite{recht2010guaranteed}. However, while \eqref{eq:bm_orig} is convex, \eqref{eq:bm_ncvx} is not. One can also show that if the solution \(\vec{Z}^*\) to \eqref{eq:bm_orig} has a rank less than the new latent dimension \(b\), the problem \eqref{eq:bm_ncvx} has no spurious local minima \cite{burer2005local}. We can also show that any stationary point \(\hat{\vec{Z}} = \hat{\vec{U}}\hat{\vec{V}}^\top \) of \eqref{eq:bm_ncvx}, achieved e.g., with gradient descent, is a global optimum for \eqref{eq:bm_orig} if it satisfies the following qualification condition
\begin{equation}
    \|\sum_{i=1}^n \nabla_{\vec{Z}}\mathcal{L}(f_i(\hat{\vec{Z}}), \labelmat_i)\|_2 \leq \beta,
\end{equation}
see \cite{mardani2013decentralized, mardani2015subspace} for the proof and more details. We thus employ the Burer-Monteiro factorization for all problems, except linear. For the transformer architectures, we choose the latent dimension such that the total model size is unchanged, whereas for the MLP architecture, we choose the latent dimension to have parameters on the same order as the other transformer architectures (e.g. \(b = sc/2\)). 
\begin{table}[t]
\caption{CIFAR-100 classification accuracy for training a single \emph{convex} head. Embeddings are generated from gMLP-S pre-trained on ImageNet. Note that the backbone is not fine-tuned.}
\label{tab:results_vit}
\vskip 0.15in
\begin{center}
\begin{small}
\begin{sc}
\begin{tabular}{lcccc}
\toprule
Convex Head & Act. & Top-1 & Top-5\\
\midrule
Self-Attention & \multirow{6}*{Linear} & \bf{67.24} & 88.63\\
MLP-Mixer & & 66.55 & \bf{88.70}\\
B-FNO & & 36.04 & 64.28\\
FNO & & 63.88 & 87.20\\
MLP & & 56.68 & 81.79\\
Linear & & 57.00 & 81.96 \\
\midrule
Self-Attention & \multirow{5}*{ReLU}& \bf{68.16} & 88.74\\
MLP-Mixer & & 67.84 & \bf{89.24}\\
B-FNO & & 66.66 & 87.93\\
FNO & & 67.87 & 88.97\\
MLP & & 64.14 & 86.57 \\
\bottomrule
\end{tabular}
\end{sc}
\end{small}
\end{center}
\vskip -0.1in
\end{table}
\section{Additional Theoretical Results}
\subsection{Visualizing the Constrained Nuclear Norm}\label{sec:constrained_nuc}
Here, we seek to provide some additional intuition around the constrained nuclear norm \(\|\vec{Z}\|_{*, \mathrm{K}}\) for a simple case, to contrast the regularization for convex ReLU against convex linear and gated ReLU networks. We provide a visualization for \(\|\vec{Z}\|_{*, \mathrm{K}}\) compared to \(\|\vec{Z}\|_*\) in the case that \(\vec{K} = \vec{I}\). This would occur in a ReLU network when \(\vec{X} = \vec{I}\) and we encounter a particular hyperplane arrangement \(\vec{D}_j = \vec{I}\). One can also note that in ReLU MLPs where \(n \leq d\) and the data \(\data\) is whitened, the convex optimization objective will reduce to a linear model with this norm as its regularization \cite{sahiner2020vector}. \\\\
In particular, in the case that \(n=d=c=2\), we can visualize the coordinates of \(\vec{Z} \in \mathbb{R}^{2 \times 2}\) which satisfy \(\|\vec{Z}\|_{*, \mathrm{K}}  \leq 1\), contrasted with the coordinates of \(\vec{Z}\) which satisfy  \(\|\vec{Z}\|_{*} \leq 1\). We do so in Figure \ref{fig:cvx-hull}, illustrating the complicated, yet still convex, regularization in the case of a ReLU neural network. 

\begin{figure*}[t!]
  \centering
  \begin{center}
      \includegraphics[width =0.75\linewidth]{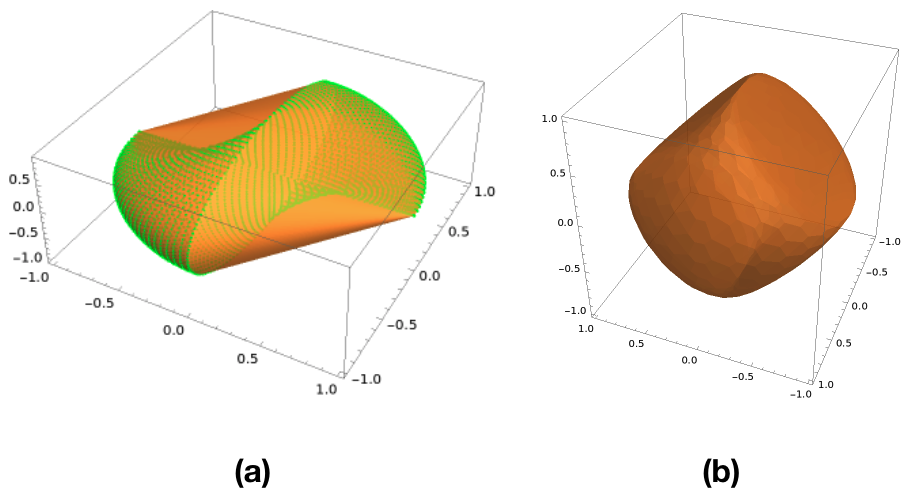}
  \end{center}
  \caption{\small Visualization of the constrained nuclear norm \(\|\vec{Z}\|_{*, \mathrm{K}}\) defined in Eq. (\ref{const_nuc_norm}). Here, we visualize \(\vec{Z} = \begin{bmatrix} z_1 & z_2\\ z_3 & z_4 \end{bmatrix} \in \mathbb{R}^{2 \times 2}\) in terms of three of its coordinates, \(z_1\), \(z_2\), and \(z_4\), with the final coordinate \(z_3\) fixed as 0. \textbf{(a)} Constrained nuclear norm space arises in the convex formulation of ReLU networks (see \eqref{const_nuc_norm}). Here, the dotted green region illustrates the set $\mathcal{C}^\prime=\{\vec{Z} = \vec{u}\vec{v}^\top: \; \vec{u} \geq 0, \, \|\vec{Z}\|_* \leq 1, \vec{v} \in \mathbb{R}^c\}$, and the opaque orange region shows its convex hull $\mathrm{conv}(\mathcal{C}^\prime)$. \textbf{(b)} In the case of linear activation, the constrained nuclear norm reduces to the standard nuclear norm \(\|\vec{Z}\|_*\). Notice that introducing ReLU nonlinearity breaks the symmetry present in \textbf{(b)} and yields a complicated non-convex space, i.e., green dots in \textbf{(a)}. However, our convex analytic approach relaxes this set as the convex hull, i.e., orange solid space in \textbf{(a)}, by keeping the extreme points, which are the points that play a crucial role in the optimal solution, intact. Therefore, we are able to convert standard non-convex ReLU network training problems into polynomial-time trainable convex optimization problems.}
  \label{fig:cvx-hull}
\end{figure*}

\subsection{Gated ReLU activation extensions}\label{sec:grelu}

\subsubsection{Self-Attention}
\begin{theorem}
For the Gated ReLU activation multi-head self-attention training problem \eqref{eq:mhsa_opt_primal}, we define 
\begin{align*}
    \vec{X} &:= \begin{bmatrix} \vec{X}_1 \otimes \vec{X}_1 \\ \cdots \\ \vec{X}_n \otimes \vec{X}_n  \end{bmatrix} \\
    \{\vec{D}_j\}_{j=1}^m &:= \{\mathrm{diag}\left(\mathbbm{1}\{\vec{X}\vec{h}_j \geq 0\}\right)\}_{j=1}^m ,
\end{align*}
for fixed gates \(\{\vec{h}_j \in \mathbb{R}^{d^2}\}_{j=1}^m\). Then, the standard non-convex training objective is equivalent to a convex optimization problem, given by
\begin{alignat}{9}\label{eq:mhsa_gated_relu_convex}
    p_{SA}^* &= \min_{\vec{Z}_j \in \mathbb{R}^{d^2 \times dc} } \sum_{i=1}^n \mathcal{L}\left(\sum_{j=1}^m \sum_{k=1}^d \sum_{\ell=1}^d \vec{G}_{i,j}^{(k, \ell)}\vec{X}_i\vec{Z}_j^{(k, \ell)}, \labelmat_i\right) + \beta \sum_{j=1}^m \|\vec{Z}_j\|_{*},
\end{alignat}
where 
\begin{align*}
    \vec{G}_{i,j} &:= (\vec{X}_i \otimes \vec{I}_{s})^\top \vec{D}_{j}^{(i)} (\vec{X}_i \otimes \vec{I}_{s}),
\end{align*} 
for \(\vec{G}_{i,j}^{(k, \ell)} \in \mathbb{R}^{s \times s}\) and \(\vec{Z}_j^{(k, \ell)} \in \mathbb{R}^{d \times c}\).
\end{theorem}
We note here that instead of the constrained nuclear norm penalty, we have a standard nuclear norm penalty on \(\vec{Z}\). 
\begin{proof}
We apply Lemmas \ref{lem:rescaling} and \ref{lem:duality} to \eqref{eq:mhsa_opt_primal} with the gated ReLU activation function to obtain
\begin{align}
 p_{SA}^* &= \max_{\dualmat_i} -\sum_{i=1}^n \mathcal{L}^*\left(\dualmat_i, \labelmat_i\right) \nonumber \\
  \mathrm{s.t.}~& \max_{\substack{\|\weightmat_{1j}\|_F\leq 1}}  \left\| \sum_{i=1}^N \dualmat_i^\top  \vec{D}_j^{(i)} \vec{X}_i \weightmat_{1j} \vec{X}_i^\top )_+ \vec{X}_i \right\|_F~\forall j \in [m]. 
\end{align}
We again apply \(\vecop(\vec{A}\vec{B}\vec{C}) = (\vec{C}^\top \otimes \vec{A})\vecop(\vec{B})\) \cite{magnus2019matrix} to obtain 
\begin{align}
 p_{SA}^* &= \max_{\dualmat_i} -\sum_{i=1}^n \mathcal{L}^*\left(\dualmat_i, \labelmat_i\right) \nonumber \\
  \mathrm{s.t.}~& \max_{\substack{\|\weight_{1j}\|_2\leq 1}}  \left\| \sum_{i=1}^N (\data_i \otimes \dualmat_i^\top) \vec{D}_j^{(i)}  (\vec{X}_i\otimes \data_i) \weight_1\right\|_2 \leq \beta ~\forall j \in [m]. 
\end{align}
Now, using the concept of dual norm, this is equal to 
\begin{align}
 p_{SA}^* &= \max_{\dualmat_i} -\sum_{i=1}^n \mathcal{L}^*\left(\dualmat_i, \labelmat_i\right) \nonumber \\
  \mathrm{s.t.}~&\max_{\substack{\|\vec{g}\|_2 \leq 1 \\ \|\weight_{1j}\|_2\leq 1 \\ j \in [m]}} \vec{g}^\top \sum_{i=1}^N (\data_i^\top  \otimes \dualmat_i^\top)  \vec{D}_{j}^{(i)}(\data_i \otimes \data_i) \weight_{1j} \leq \beta
\end{align}
Then, we have 
\begin{align}
 p_{SA}^* &= \max_{\dualmat_i} -\sum_{i=1}^n \mathcal{L}^*\left(\dualmat_i, \labelmat_i\right) \nonumber \\
  \mathrm{s.t.}~&\max_{\substack{j \in [m] \\ \|\vec{Z}\|_*}} \mathrm{trace}\left( \sum_{i=1}^n (\data_i^\top  \otimes \dualmat_i^\top)  \vec{D}_{j}^{(i)}(\data_i \otimes \data_i) \vec{Z} \right) \leq \beta
\end{align}
Now, we simply need to form the Lagrangian and solve. The Lagrangian is given by
\begin{equation}
     p_{SA}^* = \max_{\dualmat_i} \min_{\lambda \geq 0} \min_{\|\vec{Z}_j\|_* \in \mathcal{C}_j} -\sum_{i=1}^n \mathcal{L}^*(\dualmat_i, \labelmat_i) + \sum_{j=1}^m \lambda_j \left(\beta -  \sum_{i=1}^n \vecop\left(\vec{D}_{j}^{(i)} (\data_i \otimes \data_i)\vec{Z}_j\right)^\top\vecop\left(\data_i \otimes \dualmat_i \right)  \right)
\end{equation}
We now can switch the order of max and min via Sion's minimax theorem and maximize over \(\dualmat_i\). Defining \(\vec{K}_{c, s}\) as the \((c, s)\) commutation matrix \cite{magnus2019matrix}:
\begin{align*}
    \vecop(\data_i \otimes  \dualmat_i)= \left((\vec{I}_{d} \otimes \vec{K}_{c, s})(\vecop(\data_i) \otimes \vec{I}_c) \otimes \vec{I}_{s}\right)\vecop(\dualmat_i)
\end{align*}
Maximizing over \(\dualmat_i\), we have 
\begin{equation}
    p_{SA}^* = \min_{\lambda \geq 0} \min_{\|\vec{Z}_j\|_* \leq 1}  \sum_{i=1}^n \mathcal{L}\left(\sum_{j=1}^m \left((\vecop(\data_i)^\top \otimes \vec{I}_c)(\vec{I}_d \otimes \vec{K}_{s, c}) \otimes \vec{I}_{s} \right)\vecop\left(\vec{D}_{j}^{(i)} (\data_i \otimes \data_i)\lambda_j\vec{Z}_j\right), \vecop(\labelmat_i)\right) + \beta \sum_{j=1}^m \lambda_j. 
\end{equation}
Again rescaling \(\tilde{\vec{Z}}_j = \lambda_j\vec{Z}_j \), we have
\begin{equation}
    p_{SA}^* = \min_{\vec{Z}_j \in \mathbb{R}^{d^2 \times dc} }  \sum_{i=1}^n \mathcal{L}\left(\sum_{j=1}^m \left((\vecop(\data_i)^\top \otimes \vec{I}_c)(\vec{I}_d \otimes \vec{K}_{s, c}) \otimes \vec{I}_{s} \right)\vecop\left(\vec{D}_{j}^{(i)} (\data_i \otimes \data_i)\vec{Z}_j\right), \vecop(\labelmat_i)\right) + \beta \sum_{j=1}^m \|\vec{Z}_j\|_{*}.
\end{equation}
It appears as though this is a very complicated function, but it actually simplifies greatly. In particular, one can write this as 
\begin{align}
    p_{SA}^* &= \min_{\vec{Z}_j \in \mathbb{R}^{d^2 \times dc} }  \sum_{i=1}^n \mathcal{L}\left(\hat{\labelmat}_i, \labelmat_i\right) + \beta  \|\vec{Z}\|_* \nonumber \\
    \hat{\labelmat}_i[o, p] &:= \sum_{j=1}^m \sum_{k=1}^d \sum_{l=1}^d \sum_{t=1}^s \data_i[t, l]\data_i[t, k] \vec{D}_{j}^{(t, m)} \data_i[o, :]^\top \vec{Z}_j^{(k, l)}. 
\end{align}
Making any final simplifications, one obtains the desired result. 
\begin{align}
    p_{SA}^* &= \min_{\vec{Z}_j \in \mathbb{R}^{d^2 \times dc} } \sum_{i=1}^n \mathcal{L}\left(\sum_{j=1}^m \sum_{k=1}^d \sum_{\ell=1}^d \vec{G}_{i,j}^{(k, \ell)}\vec{X}_i\vec{Z}_j^{(k, \ell)}, \labelmat_i\right) + \beta \sum_{j=1}^m \|\vec{Z}_j\|_{*}. 
\end{align}
\end{proof}

\subsubsection{MLP-Mixer}
\begin{theorem}
For the Gated ReLU activation MLP-Mixer training problem \eqref{eq:mlp_mixer_primal}, we define 
\begin{align*}
    \vec{X} &:= \begin{bmatrix} \vec{X}_1^\top  \otimes \vec{I}_{s} \\ \cdots \\ \vec{X}_n^\top  \otimes \vec{I}_{s}  \end{bmatrix} \\
    \{\vec{D}_j\}_{j=1}^m &:= \{\mathrm{diag}\left(\mathbbm{1}\{\vec{X}\vec{h}_j \geq 0\}\right)\},
\end{align*}
for fixed gates \(\{\vec{h}_j \in \mathbb{R}^{s^2}\}_{j=1}^m\).
Then, the standard non-convex training objective is equivalent to a convex optimization problem, given by 
\begin{align}\label{eq:mlp_mixer_gated_relu_convex}
    p_{MM}^* &= \min_{\vec{Z}_j \in \mathbb{R}^{s^2\times dc} }  \sum_{i=1}^n \mathcal{L}\left(\begin{bmatrix} f_1(\vec{X}_i) & \cdots &  f_c(\vec{X}_i) \end{bmatrix} , \labelmat_i\right)+ \beta \sum_{j=1}^m \|\vec{Z}_j\|_{*}
\end{align}
where 
\begin{align*}
    f_p(\vec{X_i}) &:= \sum_{j=1}^m \begin{bmatrix}  \vec{D}_{j}^{(i, 1)}\vec{Z}_j^{(p, 1)} \cdots   \vec{D}_{j}^{(i, d)}\vec{Z}_j^{(p, d)}\end{bmatrix}  \vecop(\data_i)
\end{align*}
for \(\vec{D}_{j}^{(i,k)} \in \mathbb{R}^{s \times s}\) and \(\vec{Z}_j^{(p, k)} \in \mathbb{R}^{s \times s}\).
\end{theorem}

\begin{proof}
We apply Lemmas \ref{lem:rescaling} and \ref{lem:duality} to \eqref{eq:mlp_mixer_primal} with the Gated ReLU activation function to obtain
\begin{align}
    p_{MM}^* &= \max_{\vec{V}_i} -\sum_{i=1}^n\mathcal{L}^* \left (\dualmat_i,\mathbf{Y}_i \right) \nonumber \\ 
   \text{s.t.} &~ \max_{\substack{\|\weightmat_1\|_F\leq 1 \\ j \in [m]}}   \left\|\sum_{i=1}^n \dualmat_i^T \sigma_j(\weightmat_1 \data_i)_+ \right\|_F \leq \beta
\end{align}
This is equivalent to \cite{magnus2019matrix}
\begin{align}
    p_{MM}^* &= \max_{\vec{V}_i} -\sum_{i=1}^n\mathcal{L}^* \left (\dualmat_i,\mathbf{Y}_i \right) \nonumber \\ 
   \text{s.t.} &~ \max_{\substack{\|\weightmat_1\|_F\leq 1 \\ j \in [m]}}   \left\|\sum_{i=1}^n  (\vec{I}_d \otimes \dualmat_i^T) \vec{D}_j^{(i)} (\data_i^\top \otimes \vec{I}_s)\vecop(\weightmat_1) \right\|_F \leq \beta
\end{align}
Now, using the concept of dual norm, this is equal to 
\begin{align}
 p_{MM}^* &= \max_{\vec{V}_i} -\sum_{i=1}^n \mathcal{L}^* \left (\dualmat_i,\mathbf{Y}_i \right) \nonumber \\ 
   \text{s.t.} &~
    \max_{\substack{\|\vec{g}\|_2 \leq 1 \\ \|\weight_1\|_2\leq 1 \\ j \in [m]}} \vec{g}^\top\sum_{i=1}^n (\vec{I}_d \otimes \dualmat_i^T) \vec{D}_j^{(i)} (\data_i^\top \otimes \vec{I}_{s}) \weight_1 \leq \beta
\end{align}
 Then, we have 
\begin{align}
 p_{MM}^* &= \max_{\vec{V}_i} -\sum_{i=1}^n \mathcal{L}^* \left (\dualmat_i,\mathbf{Y}_i \right) \nonumber \\ 
   \text{s.t.} &~
    \max_{\substack{j \in [m] \\ \|\vec{Z}\|_* \leq 1}} \mathrm{trace}\left( \sum_{i=1}^n (\vec{I}_d \otimes \dualmat_i^T) \vec{D}_j^{(i)} (\data_i^\top \otimes \vec{I}_{s})\vec{Z} \right) \leq \beta
\end{align}
Now, we simply need to form the Lagrangian and solve. The Lagrangian is given by
\begin{equation}
    p_{MM}^*= \max_{\dualmat_i} \min_{\lambda \geq 0} \min_{\|\vec{Z}_j\|_* \leq 1} -\sum_{i=1}^n \mathcal{L}^*(\dualmat_i, \labelmat_i) + \sum_{j=1}^m \lambda_j \left(\beta -  \sum_{i=1}^n\mathrm{trace}\left( (\vec{I}_d \otimes \dualmat_i^T) \vec{D}_j^{(i)} (\data_i^\top \otimes \vec{I}_{s}) \vec{Z} \right)\right)
\end{equation}
We now can switch the order of max and min via Sion's minimax theorem and maximize over \(\dualmat_i\):
\begin{equation}
    p_{MM}^* = \min_{\lambda \geq 0} \min_{\\|\vec{Z}_j\|_* \leq 1}\max_{\dualmat_i}  -\sum_{i=1}^n \mathcal{L}^*(\dualmat_i, \labelmat_i) + \sum_{j=1}^m \lambda_j \left(\beta -  \sum_{i=1}^n \vecop\left(\vec{D}_{j}^{(i)}(\data_i^\top \otimes \vec{I}_{s})\vec{Z}_j \right)^\top \vecop\left(\vec{I}_d \otimes \dualmat_i \right)  \right)
\end{equation}
Now, defining \(\vec{K}_{c, d}\) as the \((c, d)\) commutation matrix:
\begin{align*}
    \vecop(\vec{I}_d \otimes  \dualmat_i)= \left((\vec{I}_{d} \otimes \vec{K}_{c, d})(\vecop(\vec{I}_d) \otimes \vec{I}_c) \otimes \vec{I}_{s}\right)\vecop(\dualmat_i)
\end{align*}
Solving over \(\dualmat_i\) yields
\begin{equation}
    p_{MM}^* = \min_{\lambda \geq 0}\min_{\|\vec{Z}_j\|_* \leq 1}  \sum_{i=1}^n \mathcal{L}\left(\sum_{j=1}^m \left((\vecop(\vec{I}_d)^\top \otimes \vec{I}_c)(\vec{I}_d \otimes \vec{K}_{d, c}) \otimes \vec{I}_{s} \right)\vecop\left(\vec{D}_{j}^{(i)}(\data_i^\top \otimes \vec{I}_{s})\lambda_j\vec{Z}_j \right), \vecop(\labelmat_i)\right) + \beta \sum_{j=1}^m \lambda_j
\end{equation}
Re-scaling \(\tilde{\vec{Z}}_j = \lambda_j \vec{Z}_j\) gives us
\begin{equation}
    p_{MM}^* = \min_{\vec{Z}_j}  \sum_{i=1}^n \mathcal{L}\left(\sum_{j=1}^m \left((\vecop(\vec{I}_d)^\top \otimes \vec{I}_c)(\vec{I}_d \otimes \vec{K}_{d, c}) \otimes \vec{I}_{s} \right)\vecop\left(\vec{D}_{j}^{(i)}(\data_i^\top \otimes \vec{I}_{s})\vec{Z}_j \right), \vecop(\labelmat_i)\right) + \beta \sum_{j=1}^m \|\vec{Z}_j\|_{*}.
\end{equation}
One can actually greatly simplify this result, and can re-write this as
\begin{align}
    p_{MM}^* &= \min_{\vec{Z}_j \in \mathbb{R}^{s^2\times dc} }  \sum_{i=1}^n \mathcal{L}\left(\begin{bmatrix} f_1(\vec{X}_i) & \cdots &  f_c(\vec{X}_i) \end{bmatrix} , \labelmat_i\right) + \beta \sum_{j=1}^m \|\vec{Z}_j\|_{*}\\
    f_p(\vec{X}_i) &:= \sum_{j=1}^m \begin{bmatrix}  \vec{D}_{j}^{(i, 1)}\vec{Z}_j^{(p, 1)} \cdots   \vec{D}_{j}^{(i, d)}\vec{Z}_j^{(p, d)}\end{bmatrix}  \vecop(\data_i).
\end{align}
as desired.
\end{proof}

\subsubsection{FNO}
\begin{theorem}
For the Gated ReLU activation FNO training problem \eqref{eq:fno_primal}, we define 
\begin{align*}
    \vec{X} &:= \begin{bmatrix} \mathrm{circ}(\vec{X}_1)\\ \cdots \\ \mathrm{circ}(\vec{X}_n) \end{bmatrix} \\
    \{\vec{D}_j\}_{j=1}^m &:= \{\mathrm{diag}\left(\mathbbm{1}\{\vec{X}\vec{h}_j \geq 0\}\right)\},
\end{align*}
for fixed gates \(\{\vec{h}_j \in \mathbb{R}^{sd}\}_{j=1}^m\). Then, the standard non-convex training objective is equivalent to a convex optimization problem, given by 
\begin{align}\label{eq:fno_gated_relu_convex}
    p_{FN}^* &= \min_{\vec{Z}_j \in \mathbb{R}^{sd \times c} }  \sum_{i=1}^n \mathcal{L}\left(\sum_{j=1}^m \vec{D}_j^{(i)}\mathrm{circ}(\vec{X}_i) \vec{Z}_j, \labelmat_i\right) +\beta \sum_{j=1}^m \|\vec{Z}_j\|_{*}. 
\end{align}
\end{theorem}
\begin{proof}
We now apply Lemmas \ref{lem:rescaling} and \ref{lem:duality} with the Gated ReLU activation function to obtain 
\begin{align}
    p_{FN}^* &= \max_{\dualmat_i} -\sum_{i=1}^n \mathcal{L}^*(\dualmat_i, \labelmat_i) \nonumber \\
    \mathrm{s.t.}~&\max_{\substack{\|\weight_1\|_2 \leq 1 \\ j \in [m] }} \|\sum_{i=1}^n \dualmat_i^\top \vec{D}_j^{(i)} \mathrm{circ}(\data_i) \weight_1\|_2 \leq \beta.
\end{align}
Using the concept of dual norm, this is equivalent to 
\begin{align}
    p_{FN}^* &= \max_{\dualmat_i} -\sum_{i=1}^n \mathcal{L}^*(\dualmat_i, \labelmat_i) \nonumber \\
    \mathrm{s.t.}~&\max_{\substack{\|\vec{g}\|_2 \leq 1 \\ \|\weight_1\|_2 \leq 1 \\ j \in [m]}} \vec{g}^\top \sum_{i=1}^n \dualmat_i^\top \vec{D}_j^{(i)} \mathrm{circ}(\data_i) \weight_1 \leq \beta
\end{align}
 Then, we have 
\begin{align}
    p_{FN}^* &= \max_{\dualmat_i} -\sum_{i=1}^n \mathcal{L}^*(\dualmat_i, \labelmat_i) \nonumber \\
    \mathrm{s.t.}~&\max_{\substack{ j \in [m] \\ \|\vec{Z}\|_* \leq 1}} \mathrm{trace}\left(\sum_{i=1}^n \dualmat_i^\top \vec{D}_j^{(i)} \mathrm{circ}(\data_i) \vec{Z}\right) \leq \beta
\end{align}
We form the Lagrangian as 
\begin{align}
    p_{FN}^* &= \max_{\dualmat_i}\min_{\lambda \geq 0}\min_{\|\vec{Z}_j\|_* \leq 1} -\sum_{i=1}^n \mathcal{L}^*(\dualmat_i, \labelmat_i) + \sum_{j=1}^m \lambda_j(\beta - \mathrm{trace}(\vec{Z}_j^\top \sum_{i=1}^n \vec{D}_j^{(i)} \mathrm{circ}(\data_i)^\top \dualmat_i )).
\end{align}
We switch the order of the maximum and minimum using Sion's minimax theorem and maximize over \(\dualmat_i\)
\begin{align}
    p_{FN}^* &= \min_{\lambda_j \geq 0}\min_{\|\vec{Z}_j\|_* \leq 1} \sum_{i=1}^n \mathcal{L}(\sum_{j=1}^m \vec{D}_j^{(i)} \mathrm{circ}(\data_i)\vec{Z}_j, \labelmat_i) + \beta \sum_{j=1}^m \lambda_j.
\end{align}
Lastly, we rescale \(\tilde{\vec{Z}}_j = \lambda_j \vec{Z}_j\) to obtain 
\begin{align}
    p_{FN}^* &= \min_{\vec{Z}_j \in \mathbb{R}^{sd \times c} }  \sum_{i=1}^n \mathcal{L}\left(\sum_{j=1}^m \vec{D}_j^{(i)} \mathrm{circ}(\vec{X}_i) \vec{Z}_j, \labelmat_i\right) + \beta \sum_{j=1}^m \|\vec{Z}_j\|_{*}. 
\end{align}
as desired.
\end{proof}

\subsubsection{B-FNO}
\begin{theorem}
For the Gated ReLU activation B-FNO training problem \eqref{eq:bfno_primal},  we define 
\begin{align*}
    \vec{X}_b &:= \begin{bmatrix} \mathrm{circ}(\vec{X}_1^{(b)})\\ \cdots \\ \mathrm{circ}(\vec{X}_n^{(b)}) \end{bmatrix} \\
    \{\vec{D}_{b,j}\}_{j=1}^{m} &:= \{\mathrm{diag}\left(\mathbbm{1}\{\vec{X}_b\vec{h}_{b,j} \geq 0\}\right)\},
\end{align*}
for fixed gates \(\{\vec{h}_{b,j} \in \mathbb{R}^{sd/B}\}_{j=1}^m \). Then, the standard non-convex training objective is equivalent to a convex optimization problem, given by 
\begin{align}\label{eq:bfno_gated_relu_convex}
    p_{BFN}^* &= \min_{\vec{Z}_{b, j} }  \sum_{i=1}^n \mathcal{L}\left(\begin{bmatrix} f^{(1)}(\vec{X}_i) & \cdots &  f^{(B)}(\vec{X}_i)  \end{bmatrix}, \labelmat_i\right) + \beta \sum_{b=1}^B \sum_{j=1}^{m} \|\vec{Z}_{j, b}\|_{*},
\end{align}
where
\begin{align*}
    f^{(b)} &:= \sum_{j=1}^{m} \vec{D}_{b,j}\mathrm{circ}(\vec{X}_i^{(b)})\vec{Z}_{b,j}
\end{align*}
\end{theorem}
\begin{proof}
We now apply Lemmas \ref{lem:rescaling} and \ref{lem:duality} with the Gated ReLU activation function to obtain 
\begin{align}
    p_{BFN}^* &= \max_{\dualmat_i} -\sum_{i=1}^n \mathcal{L}^*(\dualmat_i, \labelmat_i) \nonumber \\
    \mathrm{s.t.}~&\max_{\substack{\|\weight_{1b}\|_2 \leq 1 \\ b \in [B] \\ j \in [m]}} \|\sum_{i=1}^n {\dualmat_i^{(b)}}^\top  \vec{D}_{b, j}^{(i)} \mathrm{circ}(\data_i^{(b)}) \weight_{1b}\|_2 \leq \beta
\end{align}
Using the concept of dual norm, this is equivalent to 
\begin{align}
    p_{BFN}^* &= \max_{\dualmat_i} -\sum_{i=1}^n \mathcal{L}^*(\dualmat_i, \labelmat_i) \nonumber \\
    \mathrm{s.t.}~&\max_{\substack{\|\vec{g}_b\|_2 \leq 1 \\ \|\weight_{1b}\|_2 \leq 1 \\ j \in [m] \\ \vec{K}_{b, j} \weight_{1b} \geq 0}} \vec{g}^\top \sum_{i=1}^n {\dualmat_i^{(b)}}^\top \vec{D}_{b, j}^{(i)} \mathrm{circ}(\data_i^{(b)}) \weight_{1b} \leq \beta~\forall b \in [B].
\end{align} Then, we have 
\begin{align}
    p_{BFN}^* &= \max_{\dualmat_i} -\sum_{i=1}^n \mathcal{L}^*(\dualmat_i, \labelmat_i) \nonumber \\
    \mathrm{s.t.}~&\max_{\substack{ j \in [m] \\ \|\vec{Z}\|_* \leq 1}} \mathrm{trace}\left(\sum_{i=1}^n {\dualmat_i^{(b)}}^\top \vec{D}_{b, j}^{(i)} \mathrm{circ}(\data_i^{(b)}) \vec{Z}\right) \leq \beta~\forall b \in [B].
\end{align}
We form the Lagrangian as 
\begin{align}
    p_{BFN}^* &= \max_{\dualmat_i}\min_{\lambda \geq 0}\min_{\|\vec{Z}_{b,j}\|_* \leq 1} -\sum_{i=1}^n \mathcal{L}^*(\dualmat_i, \labelmat_i) + \sum_{b=1}^B \sum_{j=1}^{m} \lambda_{b,j}(\beta - \mathrm{trace}(\vec{Z}_{b,j}^\top \sum_{i=1}^n \vec{D}_{b,j}^{(i)} \mathrm{circ}(\data_i^{(b)})^\top \dualmat_i^{(b)} )).
\end{align}
We switch the order of the maximum and minimum using Sion's minimax theorem and maximize over \(\dualmat_i\)
\begin{align}
    p_{BFN}^* &= \min_{\lambda_j \geq 0}\min_{\|\vec{Z}_{b,j}\|_* \leq 1} \sum_{i=1}^n \mathcal{L}(\begin{bmatrix} \sum_{j=1}^{m}\lambda_{1, j} \vec{D}_{1,j}\mathrm{circ}(\vec{X}_i^{(1)})\vec{Z}_{1,j} & \cdots &  \sum_{j=1}^{m} \lambda_{B, j}\vec{D}_{B,j}\mathrm{circ}(\vec{X}_i^{(B)})\vec{Z}_{B,j} \end{bmatrix}, \labelmat_i) + \beta \sum_{b=1}^B \sum_{j=1}^{m}\lambda_{b,j}.
\end{align}
Lastly, we rescale \(\tilde{\vec{Z}}_{b,j} = \lambda_{b,j} \vec{Z}_{b,j}\) to obtain
\begin{align}
    p_{BFN}^* &= \min_{\vec{Z}_{b, j} }  \sum_{i=1}^n \mathcal{L}\left(\begin{bmatrix} \sum_{j=1}^{m} \vec{D}_{1,j}\mathrm{circ}(\vec{X}_i^{(1)})\vec{Z}_{1,j} & \cdots &  \sum_{j=1}^{m} \vec{D}_{B,j}\mathrm{circ}(\vec{X}_i^{(B)})\vec{Z}_{B,j} \end{bmatrix}, \labelmat_i\right) + \beta \sum_{b=1}^B \sum_{j=1}^{m} \|\vec{Z}_{j, b}\|_{*},
\end{align}
as desired.

\end{proof} 

\subsection{Additional  Attention Alternatives: PoolFormer and FNet}\label{sec:extensions}

In \cite{yu2021metaformer}, the authors propose a simple alternative to the standard MLP-Mixer architecture. In particular, the forward function is given by 
\begin{equation}
    f_{PF}(\data_i) = \sigma(\vec{P}\vec{X}_i\weightmat_1)\weightmat_2
\end{equation}
where \(\vec{P} \in \mathbb{R}^{s \times s}\) is a local pooling function. In this way, the PoolFormer architecture still mixes across different tokens, but in a non-learnable, deterministic fashion. \\\\
In \cite{lee2021fnet}, the authors propose FNet, another alternative which resembles PoolFormer architecture. In particular, a 2D FFT is applied to the input \(\data_i\) before being passed through an MLP
\begin{equation}
    f_{FNET}(\data_i) = \sigma(\vec{F}_{s}\vec{X}_i\vec{F}_d^\top \weightmat_1)\weightmat_2
\end{equation}
One can use similar convex duality results as in the main body of this paper to generate convex dual forms for this architecture for linear, ReLU, and gated ReLU activation PoolFormers and FNets. To keep these results general, we will be analyzing networks of the form 
\begin{equation}\label{eq:permuteformer_primal}
    p_{PF}^* := \min_{\weight_{1j}, \weight_{2j}} \sum_{i=1}^n \mathcal{L}\left(\sum_{j=1}^m \sigma(h(\data_i)\weight_{1j})\weight_{2j}^\top, \labelmat_i\right) + \frac{\beta}{2}\sum_{j=1}^m \|\weight_{1j}\|_2^2 + \|\weight_{2j}\|_2^2 
\end{equation}
for any generic function \(h: \mathbb{R}^{s \times d} \to \mathbb{R}^{s \times d} \), which encapsulates both methods and more. 
\begin{theorem}
For the linear activation network training problem \eqref{eq:permuteformer_primal}, for \(m \geq m^*\) where \(m^* \leq \min\{d, c\}\), the standard non-convex training objective is equivalent to a convex optimization problem, given by 
\begin{align}\label{eq:permuteformer_linear_convex}
    p_{PF}^* = \min_{\vec{Z} \in \mathbb{R}^{d \times c} }  &\sum_{i=1}^n \mathcal{L}\left(h(\data_i)\vec{Z}, \labelmat_i\right) + \beta  \|\vec{Z}\|_*. 
\end{align}
\end{theorem}
\begin{proof}
We now apply Lemmas \ref{lem:rescaling} and \ref{lem:duality} to obtain 
\begin{align}
    p_{PF}^* &= \max_{\dualmat_i} -\sum_{i=1}^n \mathcal{L}^*(\dualmat_i, \labelmat_i) \nonumber \\
    \mathrm{s.t.}~&\max_{\|\weight_1\|_2 \leq 1} \|\sum_{i=1}^n \dualmat_i^\top h(\data_i) \weight_1\|_2 \leq \beta
\end{align}
This is equivalent to 
\begin{align}
    p_{PF}^* &= \max_{\dualmat_i} -\sum_{i=1}^n \mathcal{L}^*(\dualmat_i, \labelmat_i) \nonumber \\
    \mathrm{s.t.}~& \|\sum_{i=1}^n \dualmat_i^\top h(\data_i)\|_2 \leq \beta
\end{align}
We form the Lagrangian as 
\begin{align}
    p_{PF}^* &= \max_{\dualmat_i}\min_{\lambda \geq 0}\min_{\|\vec{Z}\|_* \leq 1} -\sum_{i=1}^n \mathcal{L}^*(\dualmat_i, \labelmat_i) + \lambda(\beta - \mathrm{trace}(\vec{Z}^\top \sum_{i=1}^n h(\data_i)^\top \dualmat_i )).
\end{align}
We switch the order of the maximum and minimum using Sion's minimax theorem and maximize over \(\dualmat_i\)
\begin{align}
    p_{PF}^* &= \min_{\lambda \geq 0}\min_{\|\vec{Z}\|_* \leq 1} \sum_{i=1}^n \mathcal{L}(h(\data_i)\vec{Z}, \labelmat_i) + \beta \lambda.
\end{align}
Lastly, we rescale \(\tilde{\vec{Z}} = \lambda \vec{Z}\) to obtain 
\begin{align}
    p_{PF}^* &= \min_{\vec{Z} \in \mathbb{R}^{d \times c} }  \sum_{i=1}^n \mathcal{L}\left(h(\vec{X}_i) \vec{Z}, \labelmat_i\right) + \beta \|\vec{Z}\|_{*}. 
\end{align}
as desired.
\end{proof}

\begin{theorem}
For the ReLU activation training problem \eqref{eq:permuteformer_primal},  we define
\begin{align*}
    \vec{X} &:= \begin{bmatrix} h(\vec{X}_1)\\ \cdots \\  h(\vec{X}_n )\end{bmatrix} \\
    \{\vec{D}_j\}_{j=1}^P &:= \{\mathrm{diag}\left(\mathbbm{1}\{\vec{X}\vec{u}_j \geq 0\}\right):\;\vec{u}_j \in \mathbb{R}^{d}\},
\end{align*}
where \(P \leq 2r\left(\frac{e(n-1)}{r}\right)^r\) and \( r := \mathrm{rank}(\vec{X})\). Then, for \(m \geq m^*\) where \(m^* \leq n\min\{d, c\}\), the standard non-convex training objective is equivalent to a convex optimization problem, given by 
\begin{align}\label{eq:permuteformer_relu_convex}
    p_{PF}^* &= \min_{\vec{Z}_j \in \mathbb{R}^{d \times c} }  \sum_{i=1}^n \mathcal{L}\left(\sum_{j=1}^P \vec{D}_j^{(i)}h(\vec{X}_i) \vec{Z}_j, \labelmat_i\right) + \beta \sum_{j=1}^P \|\vec{Z}_j\|_{*, \mathrm{K}_j},
\end{align}
where
\begin{align*}
    \vec{K}_j &:= (2\vec{D}_j - \vec{I}_{ns})\vec{X}.
\end{align*}
\end{theorem}
\begin{proof}
We now apply Lemmas \ref{lem:rescaling} and \ref{lem:duality} with the ReLU activation function to obtain 
\begin{align}
    p_{PF}^* &= \max_{\dualmat_i} -\sum_{i=1}^n \mathcal{L}^*(\dualmat_i, \labelmat_i) \nonumber \\
    \mathrm{s.t.}~&\max_{\|\weight_1\|_2 \leq 1} \|\sum_{i=1}^n \dualmat_i^\top (h(\data_i) \weight_1)_+\|_2 \leq \beta
\end{align}
We introduce hyperplane arrangements \(\vec{D}_j\) and enumerate over all of them, yielding
\begin{align}
    p_{PF}^* &= \max_{\dualmat_i} -\sum_{i=1}^n \mathcal{L}^*(\dualmat_i, \labelmat_i) \nonumber \\
    \mathrm{s.t.}~&\max_{\substack{\|\weight_1\|_2 \leq 1 \\ j \in [P] \\ \vec{K}_j \weight_1 \geq 0}} \|\sum_{i=1}^n \dualmat_i^\top \vec{D}_j^{(i)} h(\data_i) \weight_1\|_2 \leq \beta.
\end{align}
Using the concept of dual norm, this is equivalent to 
\begin{align}
    p_{PF}^* &= \max_{\dualmat_i} -\sum_{i=1}^n \mathcal{L}^*(\dualmat_i, \labelmat_i) \nonumber \\
    \mathrm{s.t.}~&\max_{\substack{\|\vec{g}\|_2 \leq 1 \\ \|\weight_1\|_2 \leq 1 \\ j \in [P] \\ \vec{K}_j \weight_1 \geq 0}} \vec{g}^\top \sum_{i=1}^n \dualmat_i^\top \vec{D}_j^{(i)} h(\data_i) \weight_1 \leq \beta
\end{align}
We can also define sets \(\mathcal{C}_j := \{\vec{Z} = \vec{u}\vec{g}^\top \in \mathbb{R}^{d \times c}: \vec{K}_j\vec{u} \geq 0,\;\|\vec{Z}\|_* \leq 1\}\). Then, we have 
\begin{align}
    p_{PF}^* &= \max_{\dualmat_i} -\sum_{i=1}^n \mathcal{L}^*(\dualmat_i, \labelmat_i) \nonumber \\
    \mathrm{s.t.}~&\max_{\substack{ j \in [P] \\ \vec{Z} \in \mathcal{C}_j}} \mathrm{trace}\left(\sum_{i=1}^n \dualmat_i^\top \vec{D}_j^{(i)} h(\data_i) \vec{Z}\right) \leq \beta
\end{align}
We form the Lagrangian as 
\begin{align}
    p_{PF}^* &= \max_{\dualmat_i}\min_{\lambda \geq 0}\min_{\vec{Z}_j \in \mathcal{C}_j} -\sum_{i=1}^n \mathcal{L}^*(\dualmat_i, \labelmat_i) + \sum_{j=1}^P \lambda_j(\beta - \mathrm{trace}(\vec{Z}_j^\top \sum_{i=1}^n \vec{D}_j^{(i)} h(\data_i)^\top \dualmat_i )).
\end{align}
We switch the order of the maximum and minimum using Sion's minimax theorem and maximize over \(\dualmat_i\)
\begin{align}
    p_{PF}^* &= \min_{\lambda_j \geq 0}\min_{\vec{Z}_j \in \mathcal{C}_j} \sum_{i=1}^n \mathcal{L}(\sum_{j=1}^P \vec{D}_j^{(i)} h(\data_i)\vec{Z}_j, \labelmat_i) + \beta \sum_{j=1}^P \lambda_j.
\end{align}
Lastly, we rescale \(\tilde{\vec{Z}}_j = \lambda_j \vec{Z}_j\) to obtain 
\begin{align}
    p_{PF}^* &= \min_{\vec{Z}_j \in \mathbb{R}^{d \times c} }  \sum_{i=1}^n \mathcal{L}\left(\sum_{j=1}^P \vec{D}_j^{(i)} h(\vec{X}_i) \vec{Z}_j, \labelmat_i\right) + \beta \sum_{j=1}^P \|\vec{Z}_j\|_{*, \mathrm{K}_j}. 
\end{align}
as desired.
\end{proof}

\begin{theorem}
For the Gated ReLU activation training problem \eqref{eq:permuteformer_primal}, we define
\begin{align*}
    \vec{X} &:= \begin{bmatrix} h(\vec{X}_1)\\ \cdots \\  h(\vec{X}_n )\end{bmatrix} \\
    \{\vec{D}_j\}_{j=1}^P &:= \{\mathrm{diag}\left(\mathbbm{1}\{\vec{X}\vec{h}_j \geq 0\}\right)\},
\end{align*}
for fixed gates \(\{\vec{h}_j \in \mathbb{R}^d\}_{j=1}^m\). Then, the standard non-convex training objective is equivalent to a convex optimization problem, given by 
\begin{align}\label{eq:permuteformer_gated_relu_convex}
    p_{PF}^* &= \min_{\vec{Z}_j \in \mathbb{R}^{d \times c} }  \sum_{i=1}^n \mathcal{L}\left(\sum_{j=1}^m \vec{D}_j^{(i)}h(\vec{X}_i) \vec{Z}_j, \labelmat_i\right) + \beta \sum_{j=1}^m \|\vec{Z}_j\|_{*}. 
\end{align}
\end{theorem}
\begin{proof}
We now apply Lemmas \ref{lem:rescaling} and \ref{lem:duality} with the Gated ReLU activation function to obtain 
\begin{align}
    p_{PF}^* &= \max_{\dualmat_i} -\sum_{i=1}^n \mathcal{L}^*(\dualmat_i, \labelmat_i) \nonumber \\
    \mathrm{s.t.}~&\max_{\substack{\|\weight_1\|_2 \leq 1 \\ j \in [m] }} \|\sum_{i=1}^n \dualmat_i^\top \vec{D}_j^{(i)} h(\data_i) \weight_1\|_2 \leq \beta.
\end{align}
Using the concept of dual norm, this is equivalent to 
\begin{align}
    p_{PF}^* &= \max_{\dualmat_i} -\sum_{i=1}^n \mathcal{L}^*(\dualmat_i, \labelmat_i) \nonumber \\
    \mathrm{s.t.}~&\max_{\substack{\|\vec{g}\|_2 \leq 1 \\ \|\weight_1\|_2 \leq 1 \\ j \in [m]}} \vec{g}^\top \sum_{i=1}^n \dualmat_i^\top \vec{D}_j^{(i)} h(\data_i) \weight_1 \leq \beta
\end{align}
 Then, we have 
\begin{align}
    p_{PF}^* &= \max_{\dualmat_i} -\sum_{i=1}^n \mathcal{L}^*(\dualmat_i, \labelmat_i) \nonumber \\
    \mathrm{s.t.}~&\max_{\substack{ j \in [m] \\ \|\vec{Z}\|_* \leq 1}} \mathrm{trace}\left(\sum_{i=1}^n \dualmat_i^\top \vec{D}_j^{(i)} h(\data_i) \vec{Z}\right) \leq \beta
\end{align}
We form the Lagrangian as 
\begin{align}
    p_{PF}^* &= \max_{\dualmat_i}\min_{\lambda \geq 0}\min_{\|\vec{Z}_j\|_* \leq 1} -\sum_{i=1}^n \mathcal{L}^*(\dualmat_i, \labelmat_i) + \sum_{j=1}^m \lambda_j(\beta - \mathrm{trace}(\vec{Z}_j^\top \sum_{i=1}^n \vec{D}_j^{(i)} h(\data_i)^\top \dualmat_i )).
\end{align}
We switch the order of the maximum and minimum using Sion's minimax theorem and maximize over \(\dualmat_i\)
\begin{align}
    p_{PF}^* &= \min_{\lambda_j \geq 0}\min_{\|\vec{Z}_j\|_* \leq 1} \sum_{i=1}^n \mathcal{L}(\sum_{j=1}^m \vec{D}_j^{(i)} h(\data_i)\vec{Z}_j, \labelmat_i) + \beta \sum_{j=1}^m \lambda_j.
\end{align}
Lastly, we rescale \(\tilde{\vec{Z}}_j = \lambda_j \vec{Z}_j\) to obtain 
\begin{align}
    p_{PF}^* &= \min_{\vec{Z}_j \in \mathbb{R}^{d \times c} }  \sum_{i=1}^n \mathcal{L}\left(\sum_{j=1}^m \vec{D}_j^{(i)} h(\vec{X}_i) \vec{Z}_j, \labelmat_i\right) + \beta \sum_{j=1}^m \|\vec{Z}_j\|_{*}. 
\end{align}
as desired.
\end{proof}

\end{document}